\documentclass[11pt]{article}

\usepackage[margin=1in]{geometry}
\usepackage{amsthm}

\usepackage{amsfonts,amsmath,amssymb}
\usepackage{aliascnt}
\usepackage{mathtools,bm,mathrsfs,mleftright}
\usepackage{xcolor,enumitem,autobreak}
\usepackage{graphicx}
\allowdisplaybreaks[4]
\usepackage{multirow}
\usepackage{array,booktabs,tabularx}
\usepackage[numbers,sort&compress]{natbib}
\usepackage[colorlinks=true,linkcolor=blue,citecolor=blue,urlcolor=blue]{hyperref}
\usepackage{cleveref}

\usepackage{silence}

\theoremstyle{plain}
\newtheorem{theorem}{Theorem}[section]

\newaliascnt{lemma}{theorem}
\newtheorem{lemma}[lemma]{Lemma}
\aliascntresetthe{lemma}

\newaliascnt{proposition}{theorem}
\newtheorem{proposition}[proposition]{Proposition}
\aliascntresetthe{proposition}

\newaliascnt{corollary}{theorem}
\newtheorem{corollary}[corollary]{Corollary}
\aliascntresetthe{corollary}

\theoremstyle{definition}
\newaliascnt{definition}{theorem}
\newtheorem{definition}[definition]{Definition}
\aliascntresetthe{definition}

\newtheorem{assumption}{Assumption}

\theoremstyle{remark}
\newaliascnt{remark}{theorem}
\newtheorem{remark}[remark]{Remark}
\aliascntresetthe{remark}

\newtheorem{example}{Example}[section]

\crefname{theorem}{Theorem}{Theorems}
\Crefname{theorem}{Theorem}{Theorems}
\crefname{lemma}{Lemma}{Lemmas}
\Crefname{lemma}{Lemma}{Lemmas}
\crefname{corollary}{Corollary}{Corollaries}
\Crefname{corollary}{Corollary}{Corollaries}
\crefname{proposition}{Proposition}{Propositions}
\Crefname{proposition}{Proposition}{Propositions}
\crefname{definition}{Definition}{Definitions}
\Crefname{definition}{Definition}{Definitions}
\crefname{remark}{Remark}{Remarks}
\Crefname{remark}{Remark}{Remarks}
\crefname{example}{Example}{Examples}
\Crefname{example}{Example}{Examples}
\crefname{assumption}{Assumption}{Assumptions}
\Crefname{assumption}{Assumption}{Assumptions}
\crefname{condition}{Condition}{Conditions}
\Crefname{condition}{Condition}{Conditions}
\crefname{section}{Section}{Sections}
\Crefname{section}{Section}{Sections}
\crefname{subsection}{Subsection}{Subsections}
\Crefname{subsection}{Subsection}{Subsections}
\crefname{subsubsection}{Subsection}{Subsections}
\Crefname{subsubsection}{Subsection}{Subsections}
\crefname{figure}{Figure}{Figures}
\Crefname{figure}{Figure}{Figures}
\crefname{table}{Table}{Tables}
\Crefname{table}{Table}{Tables}

\crefformat{equation}{(#2#1#3)}
\crefrangeformat{equation}{(#3#1#4)--(#5#2#6)}
\crefmultiformat{equation}{(#2#1#3)}{ and (#2#1#3)}{, (#2#1#3)}{, and (#2#1#3)}

\newcommand{\R}{\mathbb{R}}
\newcommand{\E}{\mathbb{E}}

\newcommand{\bbP}{\mathbb{P}}

\newcommand{\T}{{\top}}
\renewcommand{\Pr}{\bbP}

\DeclareMathOperator*{\supp}{Supp}

\DeclareMathOperator{\Tr}{Tr}

\newcommand{\mca}{\mathcal}

\newcommand{\xk}[1]{\left(#1\right)}
\newcommand{\zk}[1]{\left[#1\right]}
\newcommand{\dk}[1]{\left\{#1\right\}}

\providecommand{\ang}[1]{\left\langle{#1}\right\rangle}

\providecommand{\abs}[1]{\left\lvert{#1}\right\rvert}
\providecommand{\norm}[1]{\left\lVert{#1}\right\rVert}

\providecommand{\dd}{~\mathrm{d}}

\providecommand{\bibcommenthead}{}

\hypersetup{colorlinks=true,linkcolor=blue,citecolor=blue,urlcolor=blue}

\begin{document}

\title{A Van Trees Lower Bound for Fully Interactive Differentially Private Federated Learning}

\author{
T. Tony Cai\thanks{Department of Statistics and Data Science, The Wharton School, University of Pennsylvania, Philadelphia, PA, USA. Email: \texttt{tcai@wharton.upenn.edu}.}
\and
Yicheng Li\thanks{KLATASDS-MOE, School of Statistics, East China Normal University, Shanghai, China. Email: \texttt{ycli@sfs.ecnu.edu.cn}.}
}

\date{July 8, 2026}

\maketitle

\begin{abstract}
Federated differentially private protocols can communicate over many adaptive rounds and reuse each client's local samples.
Existing lower bound arguments for federated DP are often restricted to noninteractive protocols or fresh batch decompositions, so the fundamental information-theoretic limit of estimation under fully interactive protocols remains unknown.
We establish a federated van Trees inequality for parameter estimation under squared \(\ell_2\) loss from any complete public transcript satisfying a clientwise zCDP constraint at the sample level.
A scalar trace form covers homogeneous experiments, while a matrix form preserves directional Fisher geometry in heterogeneous experiments where different clients are informative in different subspaces.
Together with existing upper bounds for the corresponding problems, these results identify the minimax rates for various statistical problems including mean estimation, linear regression, nonparametric regression, and functional mean estimation over the full class of interactive public-transcript protocols.
For these problems, arbitrary public interaction and repeated sample reuse do not improve the rate over simpler restricted protocols.
The key technical ingredient in our paper is a contraction inequality for the Fisher information in the transcript:
each client's contribution is bounded both by the Fisher information in its local experiment and by its total privacy budget.
\end{abstract}

\noindent\textbf{Keywords:} differential privacy, federated learning, van Trees inequality, lower bounds, Fisher information

\medskip

\section{Introduction}
\label{sec:introduction}

Differential privacy (DP)~\citep{dwork2006_CalibratingNoise,dwork2014_AlgorithmicFoundations} is a widely adopted rigorous mathematical framework for quantifying the privacy guarantees of randomized algorithms.
A common setting is the \emph{central} DP model, where a trusted curator has access to the entire dataset and produces a differentially private output.
In many modern applications, however, data are distributed across multiple clients, institutions, or devices, and the raw local datasets cannot be directly pooled because of privacy, legal, ownership, communication, or infrastructure constraints.
This setting is often referred to as \emph{federated} DP or differentially private federated learning.
A fundamental question in this area is to understand the tradeoffs~\citep{cai2021_CostPrivacy} between privacy and utility: how much accuracy must be sacrificed to achieve a prescribed level of privacy?

Federated learning has become an important paradigm for collaborative statistical learning from decentralized data.
Examples arise in biomedical studies across hospitals, mobile and edge-device learning, financial and commercial data analysis across organizations, multi-institution scientific studies, and privacy-sensitive public-sector applications.
In these problems, the goal is not merely to protect a centrally pooled dataset, but to enable learning across many data holders while respecting client-level constraints, heterogeneous sample sizes, different privacy requirements, and limited communication.
These features make federated learning both practically attractive and statistically distinct from classical centralized analysis.

From a theoretical and methodological perspective, federated learning introduces several new phenomena.
First, information is distributed unevenly across clients: some clients may have larger samples, stronger privacy budgets, better design distributions, or more relevant data for the target task.
Second, learning protocols can be interactive: the server may adapt later queries to earlier public messages, and each client may reuse its local samples across many rounds.
Third, the privacy constraint is naturally imposed clientwise, often on the complete public transcript rather than on a single release.
These features create statistical questions that do not appear in the same form under central DP or local DP.
In particular, one must understand how privacy budgets, sample sizes, client heterogeneity, interaction, and sample reuse jointly determine the amount of statistical information that can be extracted.

A line of work has been devoted to establishing lower bounds for the minimax risk of central DP estimation.
Early minimax formulations and testing reductions were followed by private analogues of Le Cam's method, Fano's inequality, and Assouad's lemma~\citep{barber2014_PrivacyStatistical,acharya2020_DifferentiallyPrivate}.
For high-dimensional and approximate DP problems, fingerprinting and tracing arguments provide another powerful route~\citep{bun2014_FingerprintingCodes,dwork2015_RobustTraceability,kamath2022_NewLower,peter2024_SmoothLower}, and the score attack framework turns these ideas into a general likelihood score method~\citep{cai2021_CostPrivacy,cai2023_ScoreAttack}.
More recently, Fisher information contraction and private van Trees inequalities have yielded smooth Bayesian lower bound tools under zCDP~\citep{cai2026_MinimaxAdaptive}.
These results are all focused on the central DP model.

In federated learning, the distributed and interactive nature of the protocol creates additional challenges for lower bound theory.
Recent work has studied lower bounds for federated DP in several concrete statistical problems~\citep{cai2024_FederatedNonparametric,li2024_FederatedTransfer,cai2024_OptimalFederated,auddy2024_MinimaxAdaptiveTransfer,xue2024_OptimalEstimation,hung2025_OptimalCox,cai2025_CostAdaptation}.
These results are sharp for their target problems, but their lower bound arguments are tailored to specific models and typically restricted to noninteractive releases, one-pass or chained sequential transcripts, or fresh-batch decompositions that separate communication rounds.
Such arguments do not rule out the possibility that a fully interactive protocol, with arbitrary public transcript interactions and repeated reuse of the same local samples, can achieve better performance than the restricted protocols covered by existing lower bounds.

Therefore, a central question in the theory of federated DP estimation is to determine the information-theoretic limits of estimation under the most general class of public-transcript protocols.
In this class, each client message may depend on the client's entire local dataset and on all previously public messages, the same local dataset may be reused across arbitrary adaptive rounds, and each client is charged only by its total sample-level privacy budget for the complete transcript.
To our knowledge, establishing lower bounds for this broad class of fully interactive federated DP protocols has remained an important open problem.

\subsection{Our contributions}
\label{subsec:our-contributions}

In this paper, we answer the preceding question by developing a federated van Trees lower bound for estimating finite-dimensional parameters under squared \(\ell_2\) loss from complete public transcripts.
In the formal model, the privacy constraint is imposed on the complete ordered public transcript.
This transcript includes public query choices, schedules, participation decisions, stopping rules, and public randomness, and the analysis does not require a round-by-round privacy allocation or a fresh-sample decomposition.
Equivalently, all public information that can affect later communication is part of the transcript.
This model excludes hidden trusted server state, secure aggregation or cross-client aggregation not recorded in the public transcript, and unobserved shared private randomness that can affect future public messages.

We also develop a matrix extension that retains the directional Fisher geometry of heterogeneous client experiments.
A trace bound summarizes the transcript by total Fisher information and is sufficient in homogeneous or isotropic settings, but it can lose which parameter directions are observed by which clients.
The matrix form instead decomposes the transcript Fisher information into clientwise positive semidefinite contributions, each constrained by both the client's raw Fisher information and a whitened trace budget determined by its zCDP parameter.
This directional information is needed in heterogeneous federated problems, such as missing coordinates, anisotropic random designs, and target domain prediction, where weakly covered or target-relevant directions rather than total information determine the risk.

Our van Trees inequalities are directly applicable once the Fisher information of the local data generating distributions and the prior is available.
The scalar or matrix inequality then yields a minimax risk lower bound without analyzing a particular interaction pattern.
For the problems considered below, this gives lower bounds over the full class of interactive protocols with public transcripts.
Together with existing upper bounds, these lower bounds identify the minimax rates and show that simpler one-round, noninteractive, or sample-splitting protocols remain rate optimal even after the algorithm class is enlarged to arbitrary public-transcript interaction.
\Cref{tab:known-rates-full-transcript} summarizes the resulting rates, the matching constructions, and the scope of the earlier lower bounds.
As special cases of federated DP, our results also recover known lower bounds under central or local DP.

\begin{table}[t]
  \centering
  \small
  \setlength{\tabcolsep}{3pt}
  \renewcommand{\arraystretch}{1.18}
  \begin{tabularx}{\textwidth}{@{}
    >{\raggedright\arraybackslash}p{0.28\textwidth}|
    >{\raggedright\arraybackslash}X|
    >{\raggedright\arraybackslash}p{0.17\textwidth}
    >{\raggedright\arraybackslash}p{0.17\textwidth}
    @{}}
    \toprule
    Problem
    & Minimax rate
    & Matching construction
    & Earlier lower-bound scope
    \\
    \midrule
    Gaussian mean estimation
    & \(\displaystyle
      \frac{d}{mn}
      +
      \frac{d^2}{\rho m n^2}\)
    & private weighted mean
    & non-interactive
    \\
    Linear regression
    & \(\displaystyle
      \frac{d}{mn}
      +
      \frac{d^2}{\rho m n^2}\)
    & private weighted GD
    & sample splitting
    \\
    Target domain prediction with \(q\) relevant clients, \(r=1\)
    & \(\displaystyle
      \frac{1}{qn}
      +
      \frac{1}{q\rho n^2}\)
    & \citep{li2024_FederatedTransfer}
    & non-interactive
    \\
    Target domain prediction with \(q\) relevant clients, \(r=d\)
    & \(\displaystyle
      \frac{d}{qn}
      +
      \frac{d^2}{q\rho n^2}\)
    & \citep{li2024_FederatedTransfer}
    & sample splitting
    \\
    Nonparametric regression over \(H^\alpha\) with uniform design on \([0,1]\)
    & \(\displaystyle
      (mn)^{-\frac{2\alpha}{2\alpha+1}}
      +
      (\rho m n^2)^{-\frac{\alpha}{\alpha+1}}\)
    & \citep{cai2024_OptimalFederated}
    & non-interactive
    \\
    Functional mean estimation, independent design
    & \(\displaystyle
      \frac{1}{mn}
      +
      (mkn)^{-\frac{2\alpha}{2\alpha+1}}
      +
      (\rho m k n^2)^{-\frac{\alpha}{\alpha+1}}
      +
      \frac{1}{\rho m n^2}\)
    & \citep{cai2024_OptimalFederatedFunctional,xue2024_OptimalEstimation}
    & sample splitting
    \\
    Functional mean estimation, common design
    & \(\displaystyle
      \frac{1}{mn}
      +
      k^{-2\alpha}
      +
      (\rho m n^2)^{-\frac{2\alpha}{2\alpha+1}}\)
    & \citep{cai2024_OptimalFederatedFunctional}
    & non-interactive
    \\
    \bottomrule
  \end{tabularx}
  \caption{
    Minimax rates obtained by combining the full public transcript lower
    bounds proved here with the matching constructions or references.
    Rates are shown in the homogeneous case \(n_l=n\) and \(\rho_l=\rho\).
    Upper bounds in the literature are translated to the zCDP scale with logarithmic factors ignored.
    Here \(m\) is the number of clients, \(n\) is the sample size per client,
    \(d\) is the parameter dimension, \(q\) is the number of relevant clients,
    \(\alpha\) is the smoothness parameter, and \(k\) is the number of measurements per curve.
  }
  \label{tab:known-rates-full-transcript}
\end{table}
 
\subsection{Related work}
\label{subsec:related-work}

\paragraph{Lower bounds in central DP}
The central DP lower bound literature adapts several classical minimax techniques to privacy constrained estimation.
Private versions of Le Cam, Fano, and Assouad inequalities give general tools
based on testing reductions for central DP statistical estimation~\citep{acharya2020_DifferentiallyPrivate}.
These reductions are conceptually close to classical packing and hypercube arguments, and they are especially useful when the private difficulty can be encoded through pairwise or coordinatewise testing instances.

Another major line is based on fingerprinting, tracing, and score attacks.
Fingerprinting codes were first used to prove lower bounds for approximate DP query release and empirical risk problems~\citep{bun2014_FingerprintingCodes,bassily2014_DifferentiallyPrivate}, and robust tracing arguments later removed some distributional restrictions~\citep{dwork2015_RobustTraceability}.
For statistical estimation, generalized fingerprinting lemmas have produced sharp private lower bounds for Gaussian covariance and other exponential family problems~\citep{kamath2022_NewLower,peter2024_SmoothLower}.
The score attack of \citet{cai2023_ScoreAttack} abstracts this approach through likelihood scores, recovering and extending sharp lower bounds for generalized linear models, ranking, sparse models, and nonparametric regression.
The earlier ``cost of privacy'' analysis of \citet{cai2021_CostPrivacy} used related tracing ideas for mean estimation and linear regression.
More recently, a Fisher information contraction was developed in \citet{cai2026_MinimaxAdaptive} to give a general van Trees lower bound for central DP estimation, which is especially well suited to dependent priors and matrix parameters.
All of these results are central DP results: the statistic is a single private output of a pooled sample, so they do not address clientwise privacy budgets or adaptive public transcripts.

\paragraph{Lower bounds in local DP}
The local DP model is more restrictive than central DP: each observation, user, or client must privatize its own message before it is observed by any curator.
The minimax theory of local privacy was initiated by \citet{duchi2013_LocalPrivacy,duchi2014_LocalPrivacy,duchi2018_MinimaxOptimal}, who developed private data processing inequalities and sharp minimax rates for locally private estimation.

Subsequent work has refined the information theoretic and interactive aspects of local privacy.
\citet{acharya2023_UnifiedLower} give unified lower bounds for interactive high-dimensional estimation under information constraints, a framework that includes local privacy and communication constraints.
Fisher information viewpoints have also been developed for distributed estimation under blackboard protocols and for local DP channels~\citep{barnes2019_FisherInformation,barnes2020_FisherInformation,chen2024_LqLower}.
Recent work at the user level under local DP allows each user to hold multiple observations and protects the entire local collection, establishing tight distribution estimation rates and broader phase transition lower bounds~\citep{acharya2022_DiscreteDistribution,kent2024_RateOptimality}.
Other work studies testing and communication limited local privacy~\citep{pensia2025_SimpleBinary}, nonparametric density and functional estimation under local privacy~\citep{butucea2020_LocalDifferential,butucea2022_InteractiveNoninteractive,butucea2025_NonparametricSpectral,kroll2024_NonparametricSpectral}, and high-dimensional regression under local privacy~\citep{zhu2023_ImprovedAnalysis}.

However, the techniques used in local DP lower bounds are not directly applicable to the central or federated DP setting here.
If applied naively to the central DP setting, these approaches often lose a factor of \(n\) in the effective sample size through a rough conversion by group privacy.

\paragraph{Lower bounds in federated DP}
A growing literature studies statistical and optimization lower bounds under federated or distributed DP constraints.
For convex federated learning without a trusted server, \citet{lowy2024_PrivateFederated} establish nearly tight excess risk bounds under inter-silo record level privacy and compare this trust model with local, shuffle, and central DP.
\citet{zhang2024_DifferentiallyPrivate} study how server trust affects private federated estimation and inference, while related distributed learning work considers robustness and privacy tradeoffs under distributed mechanisms~\citep{allouah2023_PrivacyrobustnessutilityTrilemma}.
These works clarify how the privacy model and trust assumptions affect achievable risk, but their lower bound tools are tailored to convex optimization or specific inferential tasks.

In nonparametric and transfer learning problems, recent papers have shown that heterogeneous client sample sizes and privacy budgets change both the optimal aggregation rule and the minimax rate.
Results include federated nonparametric regression and testing under heterogeneous distributed DP~\citep{cai2024_OptimalFederated,cai2024_FederatedNonparametric}, federated transfer learning and nonparametric classification under distributed DP~\citep{li2024_FederatedTransfer,auddy2024_MinimaxAdaptiveTransfer}, adaptive federated density estimation~\citep{cai2025_CostAdaptation}, private distributed functional data analysis~\citep{xue2024_OptimalEstimation}, federated Cox regression and cumulative hazard estimation~\citep{hung2025_OptimalCox}, and federated PCA/spiked covariance estimation~\citep{li2024_FederatedPCA}.

The present work is closest in spirit to the van Trees lower bound arguments in heterogeneous federated nonparametric and matrix estimation~\citep{cai2024_OptimalFederated,li2024_FederatedPCA,xue2024_OptimalEstimation}.
However, existing proofs either work with noninteractive local releases, use one-pass, chained sequential, or fresh batch decompositions to separate communication rounds, or exploit problem specific reductions.
The distinction here is the scope of the lower bound: each client is charged only its total sample-level zCDP budget, arbitrary public transcript interactions are allowed, and the same local dataset may be reused across adaptive rounds.

\paragraph{Notation}
For nonnegative quantities \(a\) and \(b\), we write \(a\lesssim b\) if
\(a\le Cb\) for a constant \(C\) independent of the main asymptotic parameters,
such as dimensions, sample sizes, privacy budgets, and the privacy mechanism.
We write \(a\gtrsim b\) if \(b\lesssim a\).
The implicit constants may depend on fixed regularity constants and on fixed
upper bounds imposed in the statement under consideration.
For symmetric matrices \(A\) and \(B\), \(A\preceq B\) means that \(B-A\) is
positive semidefinite, and \(A\succeq B\) is defined analogously.
For a matrix \(J\), \(J^\dagger\) denotes its Moore--Penrose pseudoinverse.
Denote by $P_A$ the orthogonal projection onto the range of a matrix \(A\).
 \section{Federated Learning Setup}
\label{sec:setting}

Differential privacy requires that the output distribution of a randomized algorithm be stable under the modification of a single sample.
Throughout the paper, adjacency is imposed at the sample level within a client: two local datasets \(X^{(l)}\) and \(\widetilde X^{(l)}\) are adjacent if they differ in exactly one observation and have the same size.
Thus the privacy unit is one sample held by a client, not the whole client dataset.
In this paper, we mainly work with zero-concentrated differential privacy (zCDP)~\citep{dwork2016_ConcentratedDifferential,bun2016_ConcentratedDifferential}, which is particularly convenient for interactive protocols.
The conversion between zCDP and the more standard \((\epsilon,\delta)\)-DP is
well understood~\citep{dwork2016_ConcentratedDifferential,bun2016_ConcentratedDifferential}:
\(\rho\)-zCDP implies
\((\rho+2\sqrt{\rho\log(1/\delta)},\delta)\)-DP for every \(\delta>0\),
and pure \(\epsilon\)-DP implies \(\epsilon^2/2\)-zCDP.
We denote by $D_\alpha(P\|Q)$ the Rényi divergence of order \(\alpha\) between distributions \(P\) and \(Q\) given by
\( D_\alpha(P\|Q)
=
\frac{1}{\alpha-1}\log\int \left(\frac{\dd P}{\dd Q}\right)^\alpha \dd Q \),
and use it similarly for random variables.

Let \(\mca D\) denote a data domain, \(\mca H\) an auxiliary input space, and \(\mca Y\) an output space.
We first give the definition of zCDP for a single mechanism~\citep{dwork2016_ConcentratedDifferential,bun2016_ConcentratedDifferential}, which is also often considered in the central DP model.

\begin{definition}[\(\rho\)-zCDP]
  \label{def:zcdp}
  A randomized mechanism \(M:\mca D\times\mca H\to\mca Y\) is called
  \(\rho\)-zero-concentrated differentially private if, for every fixed pair of
  adjacent datasets \(X,\tilde{X}\), every auxiliary input \(h\in\mca H\),
  \[
    D_\alpha\left(M(X;h)\middle\|M(\tilde{X};h)\right)
    \le
    \rho\alpha,\quad \forall \alpha > 1.
  \]
\end{definition}

\begin{figure}[!t]
  \centering
  \includegraphics[width=0.98\textwidth]{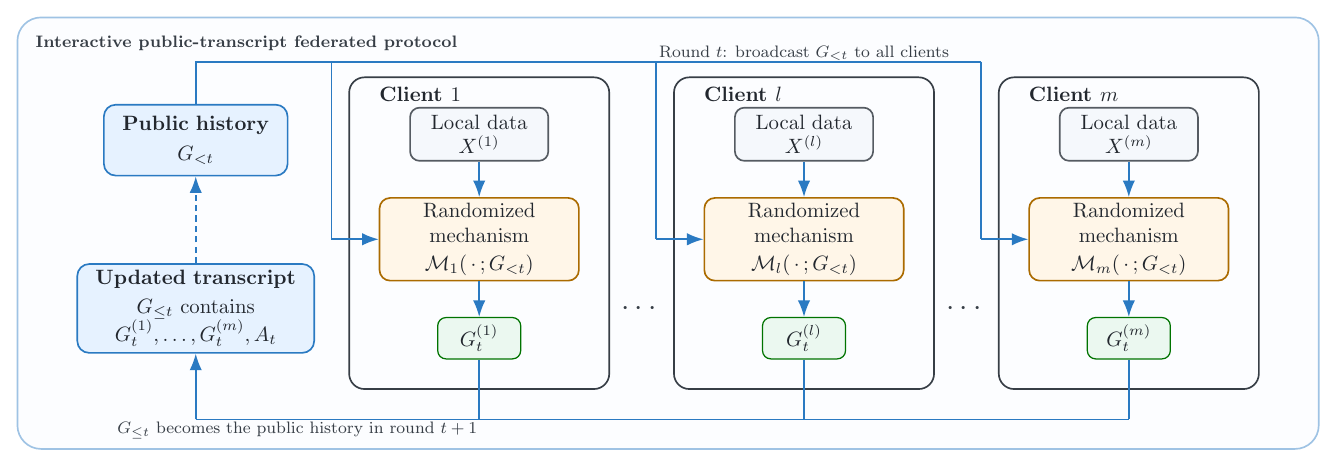}
  \caption{
    One round of an interactive federated protocol with a public transcript.
    At round \(t\), each client receives the public history \(G_{<t}\),
    applies a local randomized algorithm to its local data \(X^{(l)}\), and
    contributes \(G_t^{(l)}\); the public environment then appends the
    public record \(A_t\) at the end of the round to form the updated transcript
    \(G_{\le t}\).
  }
  \label{fig:fed-interaction-round}
\end{figure}

In federated learning, the data are distributed across multiple clients, and the algorithm may consist of multiple rounds of interaction.
Formally, suppose there are \(m\) clients, and client \(l\) holds a local dataset of size $n_{l}$:
\[
  X^{(l)}=(X_1^{(l)},\ldots,X_{n_l}^{(l)}),
  \qquad l=1,\ldots,m.
\]
The federation is coordinated by an untrusted public server.
We use the term public environment for the public coordination layer consisting
of the server, public randomization, and public post-processing records.
Each client must protect its raw data from the public environment as well as
from other clients.
Thus, each client releases only a privatized \emph{transcript} that can depend on the raw data, its own local private history, and other public information.
The transcripts are then public and can be accessed by the public server and all clients.

The learning algorithm is an interactive protocol between the clients and the
server and may be fully adaptive through the public transcript.
At a high level, the protocol starts from an initial public record, each round
broadcasts the current public history to the clients, the clients release local
public messages, and the public environment appends a public record at the end
of the round before the next round begins.
\Cref{fig:fed-interaction-round} illustrates one such round.

Write \(G_0\) for the initial public record.
Client \(l\) may also maintain a private local history \(H_t^{(l)}\),
initialized at \(H_0^{(l)}\), which is not released as part of the public
transcript.
At communication round \( t \), after observing the public history
\(G_{<t}\), client \(l\) produces a public local message and updates its
private history through a \(\theta\)-independent randomized algorithm
\begin{equation}
  \xk{G_t^{(l)},H_t^{(l)}}
  =
  \mca M_l
  \left(
    X^{(l)};G_{<t},H_{t-1}^{(l)}
  \right).
\end{equation}
After the local public messages of round \(t\) are produced, the public
environment may generate a public post-processing record \(A_t\) through a
\(\theta\)-independent public randomized algorithm
\begin{equation}
  A_t
  =
  \mca A
  \left(
    G_{<t},G_t^{(1)},\ldots,G_t^{(m)}
  \right).
\end{equation}
The public history is the collection of all transcripts
\[
  G_t=(G_t^{(1)},\ldots,G_t^{(m)},A_t),
  \qquad
  G_{<t}=G_{\le t-1} = (G_0,G_1,\ldots,G_{t-1}).
\]
Let $T$ be the total rounds of communication.
The complete ordered transcript is \( G=G_{\le T} \), and a federated learning estimator is a function \(\hat{\theta}(G)\).

We now introduce the definition of a zCDP federated protocol, which is the main privacy definition in this paper.

\begin{definition}[zCDP federated protocol]
  \label{def:full-transcript-zcdp}
  An interactive federated protocol with a public transcript is called a
  \(\bm\rho\)-zCDP federated protocol, with \(\bm\rho=(\rho_1,\ldots,\rho_m)\),
  if the complete public transcript \(G\) is $\rho_l$-zCDP as a mechanism of client \(l\)'s data \(X^{(l)}\) conditional on the other clients' data \(X^{(-l)}\).
  Namely, for every client \(l\in[m]\), every fixed value of the other clients' data \(X^{(-l)}\), every adjacent pair
  \(X^{(l)},\widetilde X^{(l)}\),
  \begin{equation}
    D_{\alpha}\xk{
    G(X^{(l)},X^{(-l)}), G(\widetilde{X}^{(l)},X^{(-l)})
    }
    \le \alpha \rho_l,\quad \forall \alpha > 1.
  \end{equation}
\end{definition}

We highlight that the privacy unit in \cref{def:full-transcript-zcdp} is one sample within a client, not the entire client dataset.
Client-level privacy, where two adjacent datasets may differ in all records held by one client, is a different and generally stronger requirement, which falls into the present sample-level formulation as the special case \(n_l=1\).

The following conventions clarify the scope of the transcript model.

\emph{Complete transcript scope.}
The privacy guarantee is imposed on the complete public transcript \(G\), not
on any single message in isolation.
Thus the definition allows arbitrary adaptivity and sample reuse across rounds.
The unreleased local states \(H_t^{(l)}\) themselves are not required to satisfy
a privacy guarantee; privacy is imposed only on what becomes public.
Inactive clients are represented by deterministic null messages.
If a nominal communication round has an internal order of local transmissions,
we refine the indexing so that \(t\) ranges over the resulting ordered public
transmissions.

\emph{Public records.}
The public records contain all public control information used by the protocol.
The initial record \(G_0\) may include initial public randomness and any initial
public query, schedule, or broadcast information.
The record \(A_t\) appended after round \(t\) contains the public computation
produced after that round, including server broadcasts, query choices, scheduling and
participation decisions, stopping indicators, public randomness, and any other
public information used in later rounds.
If the protocol uses random stopping times, data dependent schedules, or
adaptive participation rules, those decisions are part of \(G\).

\emph{Protocol kernels.}
The local algorithms \(\mca M_l\) depend only on client \(l\)'s local dataset,
its own private history, and the public history.
The public algorithm \(\mca A\) depends only on already public information.
These randomized algorithms do not explicitly depend on the unknown parameter
\(\theta\), which enters only through the sampling distribution of the
raw data.
If a protocol uses explicitly time indexed local or public algorithms, the round
counter and all time indexed instructions are included in the public history
\(G_{<t}\).
Different clients' private randomizations and local state updates are
independent conditional on the full data and public history.
No unobserved shared private randomness, hidden secure aggregation, hidden
cross client aggregation, or trusted server private state may affect future
communication unless it is included in the public transcript.

\emph{Nonexamples.}
The definition does not include a secure aggregation protocol whose hidden
partial sums or masks affect later communication without being recorded in
\(G\).
It also does not include a trusted server that stores a private state derived
from client messages and uses that state to choose later queries without adding
the state to the transcript.
Similarly, clients may not rely on shared private randomness that is unobserved
in \(G\) but influences future public messages.

Below, we give some examples of protocols satisfying this definition.
For simplicity, these examples take no unreleased local history, so
\(\mca M_l\) outputs only the public message \(G_t^{(l)}\).

\begin{example}[One pass sequential protocol]
  \label{ex:one-pass-sequential}
  Consider a one pass sequential protocol in which clients \(1,\ldots,m\) speak once in order.
  At step \(l\), client \(l\) releases
  \[
    G_l^{(l)}=\mca M_l(X^{(l)};G_{<l}),\qquad l=1,\ldots,m,
  \]
  while inactive clients release deterministic null messages and the
  public record \(A_l\) at the end of the round may be empty.
  Suppose that, for every realized public history \(g_{<l}\), the local
  mechanism \(\mca M_l(\cdot;g_{<l})\) is \(\rho_l\)-zCDP as a mechanism of
  \(X^{(l)}\).
  Then the complete transcript \(G_{\le m}\) is a \(\bm\rho\)-zCDP federated protocol.

\end{example}

\begin{example}[Roundwise protocol]
  \label{ex:roundwise-protocol}
  Consider a roundwise zCDP protocol in which
  the local mechanism \(\mca M_l(\cdot;g_{<t})\) at round \(t\) is
  \(\rho_{l,t}\)-zCDP as a mechanism of \(X^{(l)}\), conditional on every
  realized public history \(g_{<t}\).
  Then, by adaptive composition of zCDP, the complete transcript is a \(\bm\rho\)-zCDP federated protocol with \(\rho_l=\sum_{t=1}^T\rho_{l,t}\).
\end{example}

\begin{example}[General sequential protocol]
  \label{ex:ordered-sequential-protocol}
  Consider a general sequential protocol.
  Let \(a_t\in[m]\) be the active client at transmission \(t\), and write
  \[
    G_t^{(a_t)}=\mca M_{a_t}(X^{(a_t)};G_{<t}),\qquad t=1,\ldots,T,
  \]
  with deterministic null messages from inactive clients and with any
  public post-processing at the end of the round included in \(A_t\).
  Suppose that the local mechanism \(\mca M_{a_t}(\cdot;g_{<t})\) is
  \(\rho_{a_t,t}\)-zCDP as a mechanism of \(X^{(a_t)}\), conditional on every
  realized public history \(g_{<t}\).
  This is a special case of \cref{ex:roundwise-protocol}: in each transmission,
  all inactive clients send deterministic null messages, equivalently \(\rho_{l,t}=0\) for \(l\ne a_t\).
  Hence the complete ordered transcript is a \(\bm\rho\)-zCDP federated protocol with
  \( \rho_l=\sum_{t:a_t=l}\rho_{l,t}\).
\end{example}
 \section{General Federated van Trees Lower Bounds}
\label{sec:general-van-trees}

This section presents the paper's main lower bound theorems for fully
interactive zCDP federated protocols.
We begin with a scalar trace form for homogeneous client experiments,
which is the simplest version and the one used for most rate calculations.
We then give a matrix extension for heterogeneous experiments, where different
clients may be informative in different parameter directions.
For presentation, we defer the regularity conditions to
\cref{ass:van-trees-regularity}, including the standard differentiability,
integrability, boundary, and transcript regularity conditions needed to apply
the classical van Trees inequality.

\subsection{Scalar trace bound for homogeneous clients}
\label{subsec:trace-fed-vantrees}

We start with the homogeneous experiment, where all clients sample from the
same parametric model but may have different sample sizes and privacy budgets.

Client \(l\) holds independent local observations
\begin{equation}
  \label{eq:homogeneous-independent-client-model}
  X_i^{(l)}\stackrel{\mathrm{i.i.d.}}{\sim}P_\theta,
  \qquad
  i=1,\ldots,n_l,\quad l=1,\ldots,m,
  \qquad
  \theta\in\Theta\subseteq\R^p,
\end{equation}
and the local datasets are independent across clients.
Let \(p_\theta\) be a density of \(P_\theta\) with respect to a
\(\theta\)-independent dominating measure.
Define
\[
  s(x;\theta)=\nabla_\theta\log p_\theta(x),
  \qquad
  I_x(\theta)=\E_\theta[s(X;\theta)s(X;\theta)^\T]
\]
to be the score and Fisher information matrix of one observation.
For comparison with the heterogeneous notation used below, this is the case
\(s_{li}=s\), with
\[
  S_l(\theta)=\sum_{i=1}^{n_l}s(X_i^{(l)};\theta),
  \qquad
  J_l(\theta)=n_l I_x(\theta).
\]
The complete public transcript is denoted by \(G\), and all estimators in this
section are functions of \(G\).
The notation \(\E_\pi\E_\theta\) means that \(\theta\sim\pi\) is drawn first
and the data and transcript are then generated under \(P_\theta\).

For a clientwise zCDP budget \(\bm\rho=(\rho_1,\ldots,\rho_m)\), write
\[
  \kappa_l=e^{2\rho_l}-1,\qquad l=1,\ldots,m.
\]
The first ingredient is an information contraction inequality for the complete
public transcript.
It is the only point at which the interactive protocol and the clientwise zCDP
constraint enter the scalar lower bound.

\begin{proposition}[Full transcript trace information contraction]
  \label{prop:trace-transcript-information-contraction}
  Under model \cref{eq:homogeneous-independent-client-model} and
  \cref{ass:van-trees-regularity},
  let \(G\) be the complete public transcript of a \(\bm\rho\)-zCDP federated protocol.
  Then, for every fixed \(\theta\),
  \begin{equation}
    \label{eq:trace-transcript-information-contraction}
    \Tr I_G(\theta)
    \le
    \sum_{l=1}^m
    \left[
      \kappa_l n_l^2\norm{I_x(\theta)}_{\mathrm{op}}
      \wedge
      n_l\Tr I_x(\theta)
    \right].
  \end{equation}
\end{proposition}

The two terms in each clientwise summand have different origins.
The term \(n_l\Tr I_x(\theta)\) is the ordinary Fisher information available in
client \(l\)'s raw sample.
The term \(\kappa_l n_l^2\norm{I_x(\theta)}_{\mathrm{op}}\) is the privacy
upper bound on the Fisher information that can pass through the complete
public transcript under the client \(l\) zCDP budget at the sample level.
The proposition charges the single complete transcript budget \(\rho_l\).
It is therefore compatible with arbitrary adaptivity and repeated reuse of the
same local observations across public rounds.

Combining this contraction with the multivariate van Trees inequality gives the
main scalar lower bound.

\begin{theorem}[Scalar federated van Trees inequality]
  \label{thm:general-fed-dp-vantrees}
  Under model \cref{eq:homogeneous-independent-client-model} and
  \cref{ass:van-trees-regularity}, every estimator
  \(\hat\theta=\hat\theta(G)\) based on a \(\bm\rho\)-zCDP federated protocol
  satisfies
  \[
    \E_\pi\E_\theta\norm{\hat\theta-\theta}_2^2
    \ge
    \frac{p^2}{
      \int_\Theta
      \sum_{l=1}^m
      \left[
        \kappa_l n_l^2\norm{I_x(\theta)}_{\mathrm{op}}
        \wedge
        n_l\Tr I_x(\theta)
      \right]
      \dd\pi(\theta)
      +
      \Tr J_\pi
    }.
  \]
\end{theorem}

For the bounded privacy regimes used in the applications, the factor \(\kappa_l=e^{2\rho_l}-1\) is equivalent to \(\rho_l\) up to constants and the bound can be written in a more familiar form.

\begin{corollary}[Bounded privacy form]
  \label{cor:bounded-rho-fed-dp-vantrees}
  Under model \cref{eq:homogeneous-independent-client-model} and
  \cref{ass:van-trees-regularity}, suppose
  \(0\le \rho_l\le C\) for all \(l\).
  Then
  \[
    \E_\pi\E_\theta\norm{\hat\theta-\theta}_2^2
    \gtrsim
    \frac{p^2}{
      \int_\Theta
      \sum_{l=1}^m
      \left[
        \rho_l n_l^2\norm{I_x(\theta)}_{\mathrm{op}}
        \wedge
        n_l\Tr I_x(\theta)
      \right]
      \dd\pi(\theta)
      +
      \Tr J_\pi
    }.
  \]
\end{corollary}

\subsection{Matrix extension for heterogeneous clients}
\label{subsec:matrix-heterogeneous-vantrees}

The scalar theorem deliberately compresses the transcript Fisher information into its trace.
This is appropriate when the coordinates are statistically comparable and the loss treats all directions symmetrically.
In heterogeneous experiments, however, different clients may observe different subspaces or have anisotropic designs.
Then total Fisher information is not enough: the lower bound must also record where that information is located.
The matrix extension keeps this directional information.

Consider independent, not necessarily identically distributed, local observations
\begin{equation}
  \label{eq:heterogeneous-independent-client-model}
  X_i^{(l)}\sim P_{\theta,li},
  \qquad
  i=1,\ldots,n_l,\quad l=1,\ldots,m,
  \qquad
  \theta\in\Theta\subseteq\R^p.
\end{equation}
Let
\[
  s_{li}(X_i^{(l)};\theta)
  =
  \nabla_\theta\log p_{\theta,li}(X_i^{(l)})
\]
be the local score, and define
\[
  S_l(\theta)=\sum_{i=1}^{n_l}s_{li}(X_i^{(l)};\theta),
  \qquad
  J_l(\theta)=\E_\theta[S_l(\theta)S_l(\theta)^\T].
\]
If the observations are i.i.d.\ within client \(l\), then
\(J_l(\theta)=n_l I_l(\theta)\), where \(I_l(\theta)\) is the Fisher
information of one sample from that client.
Thus \(J_l(\theta)\) is the raw Fisher information available at client \(l\)
before privacy is imposed.

The next theorem decomposes the transcript Fisher information into clientwise
positive semidefinite contributions.
Each contribution is bounded by the client's raw information matrix and by a
whitened trace budget determined by the zCDP constraint.
The proof is given in
\cref{subsec:proof-heterogeneous-matrix-information-contraction}.

\begin{theorem}[Matrix transcript contraction]
  \label{thm:heterogeneous-matrix-information-contraction}
  Under the heterogeneous client model in
  \cref{eq:heterogeneous-independent-client-model} and
  \cref{ass:van-trees-regularity},
  let \(G\) be the complete public transcript of a
  \(\bm\rho\)-zCDP federated protocol.
  For every fixed \(\theta\), there exist positive semidefinite matrices
  \(A_l(\theta)\), \(l=1,\ldots,m\), such that
  \[
    I_G(\theta)=\sum_{l=1}^m A_l(\theta),
  \]
  and
  \[
    0\preceq A_l(\theta)\preceq J_l(\theta),
    \qquad
    \Tr\zk{J_l(\theta)^\dagger A_l(\theta)}
    \le
    \kappa_l n_l.
  \]
  Equivalently,
  \[
    A_l(\theta)
    =
    J_l(\theta)^{1/2}B_l(\theta)J_l(\theta)^{1/2},
  \]
  where
  \[
    0\preceq B_l(\theta)\preceq P_{J_l(\theta)},
    \qquad
    \Tr B_l(\theta)\le \kappa_l n_l.
  \]
\end{theorem}

Here \(A_l(\theta)\) is the part of the transcript Fisher information
attributable to client \(l\).
The Loewner bound \(A_l(\theta)\preceq J_l(\theta)\) says that privacy and
interaction cannot create more information than the client's raw sample.
The whitened trace bound says that, after measuring information in the
geometry of \(J_l(\theta)\), the public transcript can pass through at most
\(\kappa_l n_l\) Fisher information units from client \(l\).

Moreover, the matrix contraction is a strict strengthening of the scalar trace.
Indeed, noticing that
\[
  \Tr A_l=\Tr(J_l B_l)\le\norm{J_l}_{\mathrm{op}}\Tr B_l\le\kappa_l n_l\norm{J_l}_{\mathrm{op}}
\]
and that \(A_l\preceq J_l\) gives \(\Tr A_l\le\Tr J_l\),
taking traces in \cref{thm:heterogeneous-matrix-information-contraction}
gives the heterogeneous trace consequence
\[
  \Tr I_G(\theta)
  \le
  \sum_{l=1}^m
  \left[
    \kappa_l n_l\norm{J_l(\theta)}_{\mathrm{op}}
    \wedge
    \Tr J_l(\theta)
  \right].
\]
In the homogeneous model, \(J_l(\theta)=n_l I_x(\theta)\), so this reduces to
\cref{eq:trace-transcript-information-contraction}.

Combining the matrix van Trees inequality with
\cref{thm:heterogeneous-matrix-information-contraction} gives the following
matrix federated lower bound.
In the statement, \(P_{J_l(\theta)}\) denotes the orthogonal projection onto
\(\operatorname{range}(J_l(\theta))\).
Thus, if \(J_l(\theta)\) is singular, \(B_l(\theta)\) is constrained to act
only on the directions in which client \(l\)'s raw experiment has Fisher
information.
The final inverse is interpreted in the extended sense: zero eigenvalues
contribute \(+\infty\).

\begin{theorem}[Matrix federated van Trees inequality]
  \label{thm:matrix-fed-dp-vantrees}
  Under the heterogeneous client model in
  \cref{eq:heterogeneous-independent-client-model} and
  \cref{ass:van-trees-regularity}, every estimator
  \(\hat\theta=\hat\theta(G)\) satisfies
  \[
    \E_\pi\E_\theta\norm{\hat\theta-\theta}_2^2
    \ge
    \inf_{\{B_l(\theta)\}}
    \Tr
    \left[
      J_\pi
      +
      \int_\Theta
      \sum_{l=1}^m
      J_l(\theta)^{1/2}B_l(\theta)J_l(\theta)^{1/2}
      \dd\pi(\theta)
    \right]^{-1},
  \]
  where the infimum is over measurable choices satisfying
  \[
    0\preceq B_l(\theta)\preceq P_{J_l(\theta)},
    \qquad
    \Tr B_l(\theta)\le \kappa_l n_l.
  \]
  The trace inverse is interpreted in the extended sense: zero eigenvalues
  contribute \(+\infty\).
\end{theorem}

Because the theorem covers heterogeneous experiments and keeps the lower bound in matrix form,
the expression is necessarily more complex than the scalar trace bound and, in general, does not reduce to a simpler closed form.
To apply it, one first computes the raw client Fisher information
\(J_l(\theta)\).
The matrix \(B_l(\theta)\) then describes the part of client \(l\)'s
information that passes through the transcript after whitening by the geometry
of \(J_l(\theta)\).
The trace constraint
\(\Tr B_l(\theta)\le(e^{2\rho_l}-1)n_l\) says that client \(l\) can contribute
at most this many effective whitened directions.
After these clientwise contributions are combined, the resulting information
matrix is inserted into the matrix van Trees expression, or into its
restriction to the loss-relevant subspace when the loss is projected.

The infimum removes protocol-dependent quantities.
For a fixed protocol, the proof produces a particular feasible family
\(B_l^G(\theta)\) induced by the transcript Fisher information.
Using that family gives a protocol-specific lower bound.
Taking the infimum over the larger feasible class can only weaken this
protocol-specific bound, but it yields a uniform lower bound that depends only
on the local Fisher geometries and the clientwise zCDP budgets.
We will give further interpretations and applications of this matrix bound in
the discussion below.

\subsection{Discussion}
\label{subsec:general-vantrees-discussion}

\subsubsection{Rate consequences}
\label{subsubsec:rate-consequences}
\Cref{thm:general-fed-dp-vantrees,cor:bounded-rho-fed-dp-vantrees} reduce the
homogeneous federated lower bound calculation to the Fisher information
\(I_x(\theta)\) of one observation.
Suppose, for example, that
\(\norm{I_x(\theta)}_{\mathrm{op}}\le\sigma^{-2}\) uniformly and the prior
Fisher information is dominated by the transcript information term.
Then the bounded privacy form gives
\[
  \E_\pi\E_\theta\norm{\hat\theta-\theta}_2^2
  \gtrsim
  \sigma^2
  \zk{
    \sum_{l=1}^m
    \left(
      \frac{p}{n_l}
      \vee
      \frac{p^2}{\rho_l n_l^2}
    \right)^{-1}
  }^{-1}.
\]
In the homogeneous case \(n_l=n\) and \(\rho_l=\rho\), this becomes
\[
  \E_\pi\E_\theta\norm{\hat\theta-\theta}_2^2
  \gtrsim
  \sigma^2
  \left(
    \frac{p}{mn}
    \vee
    \frac{p^2}{\rho mn^2}
  \right).
\]
When \(m=1\), the expression reduces to the central DP van Trees form based on
Fisher information contraction~\citep{cai2026_MinimaxAdaptive}.
When \(n_l=1\) for all \(l\), the privacy limited term matches the usual
high privacy local DP Fisher information scaling studied in local privacy
lower bounds~\citep{duchi2018_MinimaxOptimal,acharya2023_UnifiedLower}.

\subsubsection{Effect of interaction}
\label{subsubsec:effect-of-interaction}
\Cref{thm:general-fed-dp-vantrees,thm:matrix-fed-dp-vantrees} are uniform over
the full class of interactive protocols with public transcripts in
\cref{def:full-transcript-zcdp}.
They apply to protocols with a single local release, ordered sequential
protocols, and protocols with multiple adaptive rounds in which the same local
data may be reused across rounds.
The bounds depend on the protocol only through the clientwise zCDP budgets
\(\rho_l\) for the complete transcript, not through the number of rounds, the
order of messages, or the way in which a budget is split across rounds.
For an explicitly roundwise protocol, this means that the lower bound applies
after the usual zCDP composition \(\rho_l=\sum_t\rho_{l,t}\); more generally,
no roundwise decomposition is required once the complete transcript satisfies
\cref{def:full-transcript-zcdp}.

This uniformity is useful for interpreting matching constructions.
If a one round or noninteractive construction attains one of the rates lower
bounded by the scalar or matrix inequality, then allowing arbitrary interaction
over multiple rounds cannot improve that rate.
In this sense, a simple protocol that matches the bound is rate optimal not
only within its own restricted communication structure but also within the
larger class of fully interactive protocols with public transcripts.
Interaction may still matter for constants, communication, numerical stability, or adaptation to unknown tuning parameters,
but it cannot overcome the transcript Fisher information limits in
\cref{thm:general-fed-dp-vantrees,thm:matrix-fed-dp-vantrees}.

The main technical step behind this uniformity is the clientwise score
decomposition for the complete transcript.
Existing lower bounds are proved only for noninteractive releases, fresh batch
protocols, or sample splitting reductions, and they do not by themselves rule
out a fully adaptive algorithm that repeatedly queries the same local datasets.
Here the entire ordered public transcript \(G\), including all adaptive
messages and public post-processing records, is the statistic to which the
information contraction is applied.
Conditioning on a realized complete transcript fixes the public adaptive
choices, while the unreleased private histories remain local to their clients.
This gives posterior factorization and hence the clientwise score projections
used in \cref{prop:trace-transcript-information-contraction,thm:heterogeneous-matrix-information-contraction}.
Thus arbitrary interaction and sample reuse are included in the proof
architecture rather than handled by a reduction to fresh batches or
noninteractive releases.

\subsubsection{The matrix formulation}
\label{subsubsec:matrix-formulation}
The measurable choice condition in \cref{thm:matrix-fed-dp-vantrees} is a way
to write the sharpest integrated upper envelope on possible transcript
information matrices.
For applications, one may either solve the matrix optimization directly or
lower bound its value using simpler consequences of the constraints.
Enlarging the feasible class gives a weaker but still valid lower bound,
whereas evaluating a single feasible family gives only an upper bound on the
infimum.
Another common route is to replace each feasible contribution
\(J_l^{1/2}B_lJ_l^{1/2}\) by an explicit Loewner upper envelope or to project
the expression to the subspace appearing in the loss.
These relaxations preserve validity because the trace of the inverse is
order reversing on the positive semidefinite cone.

When the relevant Fisher information matrices commute, the optimization has a
coordinatewise interpretation.
Suppose, for instance, that \(J_\pi\) and \(J_l(\theta)\) are independent of
\(\theta\) and are diagonal in the same orthonormal basis:
\[
  J_\pi=\operatorname{diag}(j_{\pi 1},\ldots,j_{\pi p}),
  \qquad
  J_l=\operatorname{diag}(j_{l1},\ldots,j_{lp}).
\]
Then \cref{thm:matrix-fed-dp-vantrees} yields
\[
  \E_\pi\E_\theta\norm{\hat\theta-\theta}_2^2
  \ge
  \inf_{\{b_{lj}\}}
  \sum_{j=1}^p
  \frac{1}{
    j_{\pi j}+\sum_{l=1}^m j_{lj}b_{lj}
  },
\]
where
\[
  0\le b_{lj}\le \mathbf 1\{j_{lj}>0\},
  \qquad
  \sum_{j=1}^p b_{lj}\le \kappa_l n_l
  \quad\text{for each }l.
\]
Thus \(b_{lj}\) is the fraction of client \(l\)'s available information budget
allocated to direction \(j\), and the constraint
\(\sum_j b_{lj}\le\kappa_l n_l\) is the clientwise privacy information budget.

Noncommuting Fisher geometries do not admit such a coordinatewise reduction.
For a simple example, take \(p=3\), \(J_\pi=\lambda I_3\), and two clients with
\[
  J_1=n_1\sigma^{-2}P_1,
  \qquad
  J_2=n_2\sigma^{-2}P_2,
\]
where \(P_1\) is the projection onto \(\operatorname{span}\{e_1,e_2\}\) and
\(P_2\) is the projection onto
\(\operatorname{span}\{(e_1+e_3)/\sqrt2,e_2\}\).
Then \(P_1P_2\ne P_2P_1\), so there is no common orthonormal basis in which the
two clients' allocations reduce to independent scalars \(b_{lj}\).
If \(\kappa_l n_l\le1\), each client can pass through at most one effective
whitened direction, but the two available planes overlap obliquely.
In this case the matrix formulation records how privacy constrained
information can be placed across nonaligned subspaces, while a scalar trace
calculation records only total Fisher mass.

\subsubsection{Directional information}
\label{subsubsec:directional-information}
The scalar trace bound is most informative when the loss treats directions
symmetrically and clients have comparable Fisher geometries.
The matrix bound is needed when the relevant difficulty is directional.
Projection experiments make this point explicit because the feasible set can be
read directly from the observed coordinate subsets.
If client \(l\) has Fisher information \(J_l=n_l\sigma_l^{-2}P_l\), where
\(P_l\) is the orthogonal projection onto a coordinate subset
\(\mca S_l\subseteq[p]\), and
\(J_\pi=\operatorname{diag}(j_{\pi1},\ldots,j_{\pi p})\), then
\cref{thm:matrix-fed-dp-vantrees} gives
\[
  \E_\pi\E_\theta\norm{\hat\theta-\theta}_2^2
  \ge
  \inf_{\{b_{lj}\}}
  \sum_{j=1}^p
  \frac{1}{
    j_{\pi j}
    +
    \sum_{l:\,j\in \mca S_l} n_l\sigma_l^{-2}b_{lj}
  },
\]
where
\[
  0\le b_{lj}\le 1,
  \qquad
  \sum_{j\in \mca S_l} b_{lj}\le \kappa_l n_l.
\]
Only clients observing coordinate \(j\) can contribute to the denominator for
that coordinate.
The next balanced case gives a closed form consequence.

\begin{corollary}[Balanced coordinate coverage]
  \label{cor:balanced-coordinate-coverage}
  Under the conditions of \cref{thm:matrix-fed-dp-vantrees}, suppose \(n_l=n\), \(\rho_l=\rho\), and
  \(J_l=n\sigma^{-2}P_l\), where each \(P_l\) is a coordinate projection of
  rank \(r\) onto a coordinate subset \(\mca S_l\subseteq[p]\).
  Assume the coordinate coverage is balanced:
  \[
    \sum_{l=1}^m P_l=qI_p,
    \qquad
    mr=pq.
  \]
  If \(J_\pi\preceq \lambda_\pi I_p\), then
  \[
    \E_\pi\E_\theta\norm{\hat\theta-\theta}_2^2
    \ge
    \frac{p}{
      \lambda_\pi
      +
      \sigma^{-2}qn
      \left(1\wedge\frac{\kappa n}{r}\right)
    },
    \qquad
    \kappa=e^{2\rho}-1.
  \]
  In particular, when the prior information is no larger than the transcript
  information term,
  \[
    \E_\pi\E_\theta\norm{\hat\theta-\theta}_2^2
    \gtrsim
    \sigma^2
    \left(
      \frac{p}{qn}
      \vee
      \frac{pr}{q\kappa n^2}
    \right).
  \]
\end{corollary}

The proof is given in
\cref{subsec:proof-heterogeneous-matrix-information-contraction}.
Here \(q\) is the number of clients covering each coordinate, while \(r\) is
the number of coordinates over which each client must spend its
privacy information budget.
The privacy contribution therefore produces the additional term
\(pr/(q\kappa n^2)\).
This is the form used in missing feature and target subspace examples.
It also illustrates why a scalar trace denominator can be misleading in
heterogeneous problems: trace information counts total Fisher mass, but it
does not identify whether that mass lies in directions relevant to the loss.
For instance, if every \(J_l\) is supported on \(\operatorname{span}(e_1)\) and
the target risk depends on \(e_2^\T\theta\), then the matrix bound leaves the
target direction controlled only by the prior, whereas a trace only expression
sees positive total Fisher information in an irrelevant direction.

\subsubsection{Extensions to other DP notions}
\label{subsubsec:extensions-other-dp-notions}
The lower bounds can also be read for other common differential privacy
notions.
For pure DP, an \(\epsilon_l\)-DP client mechanism is
\(\epsilon_l^2/2\)-zCDP, so the same bounds hold after replacing
\(\rho_l\) by \(\epsilon_l^2/2\).
For approximate \((\epsilon_l,\delta_l)\)-DP, there is no direct conversion to zCDP, but the proof changes only in the single-client contraction step, see
\Cref{app:extension-approximate-dp}.
The main difference is that the zCDP contraction lemma is replaced by a contraction with a truncation remainder, which accounts for the additional failure events in approximate DP.
We give the homogeneous scalar form explicitly in
\cref{thm:homogeneous-approx-dp-vantrees}
for simplicity, but the heterogeneous matrix version is also straightforward.

 \section{Applications}
\label{sec:applications}

This section applies
\cref{thm:general-fed-dp-vantrees,thm:matrix-fed-dp-vantrees} to several
standard estimation problems.
The examples separate two uses of the theory: the scalar inequality gives
isotropic rates, while the matrix inequality records directions that are missing
or weakly covered across clients.
Let \(\mca{M}_{\bm{\rho}}\) denote the class of estimators based on
\(\bm\rho\)-zCDP federated protocols in \cref{def:full-transcript-zcdp}, where
interactions are fully allowed.
Throughout this section, the privacy budgets are assumed bounded,
\(0\le\rho_l\le C\) for a fixed constant \(C\).
We give lower bounds on the minimax risk over \(\mca{M}_{\bm{\rho}}\) for each
problem.
For the homogeneous case \(n_l=n\) and \(\rho_l=\rho\),
\cref{tab:known-rates-full-transcript} summarizes the sharp rates obtained by
combining these lower bounds with matching constructions.
Detailed derivations of all displayed rates are given in
\cref{app:application-derivations}.

\subsection{Mean Estimation}
\label{subsec:app-mean-estimation}

We begin with mean estimation, where the scalar and matrix bounds can be seen
in their simplest forms.
The Gaussian model illustrates the isotropic rate, while the coordinate
projection model shows how uneven coverage across directions changes the lower
bound.

\subsubsection{Gaussian mean estimation}
\label{subsubsec:app-gaussian-mean-estimation}

Suppose client \(l\) observes
\begin{equation}
  \label{eq:app-gaussian-mean-model}
  X_i^{(l)}\sim \mca{N}(\theta,\sigma^2 I_d),
  \qquad i=1,\ldots,n_l,
\end{equation}
independently across \(i\) and \(l\).

\begin{proposition}[Gaussian mean estimation]
  \label{prop:app-gaussian-mean}
  Under the model \cref{eq:app-gaussian-mean-model}, we have
  \begin{equation}
    \label{eq:app-gaussian-mean-rate}
    \inf_{\hat\theta \in \mca{M}_{\bm{\rho}}}
    \sup_{\theta\in [-1,1]^d}
    \E_\theta\norm{\hat\theta-\theta}_2^2
    \gtrsim
    d
    \wedge
    \sigma^2
    \zk{
      \sum_{l=1}^m
      \left(
        \frac{d}{n_l}
        \vee
        \frac{d^2}{\rho_l n_l^2}
      \right)^{-1}
    }^{-1}.
  \end{equation}
\end{proposition}

The two terms inside the clientwise maximum correspond respectively to the
ordinary sampling difficulty and the privacy contraction difficulty.
The calculation is given in \cref{app:deriv-gaussian-mean}.

\subsubsection{Missing features and coordinate coverage}
\label{subsubsec:app-missing-features}

The next example uses the matrix inequality to track how information is spread
across observed coordinate directions.
Suppose client \(l\) observes only the coordinate projection
\begin{equation}
  \label{eq:app-missing-features-model}
  X_i^{(l)}=P_l\theta+\xi_i^{(l)},
  \qquad
  \xi_i^{(l)}\sim \mca{N}(0,\sigma^2P_l)
\end{equation}
on the range of \(P_l\), where \(P_l\) has rank \(s\).
Assume \(n_l=n\), \(\rho_l=\rho\), and balanced coverage
\begin{equation}
  \label{eq:app-missing-features-balanced-coverage}
  \sum_{l=1}^m P_l=qI_d,
  \qquad
  ms=dq .
\end{equation}

\begin{proposition}[Missing features with balanced coverage]
  \label{prop:app-missing-features}
  Under the model \cref{eq:app-missing-features-model} and balanced coverage
  \cref{eq:app-missing-features-balanced-coverage}, we have
  \begin{equation}
    \label{eq:app-missing-features-rate}
    \inf_{\hat\theta\in \mca{M}_{\bm{\rho}}}
    \sup_{\theta\in [-1,1]^d}
    \E_\theta\norm{\hat\theta-\theta}_2^2
    \gtrsim
    d
    \wedge
    \sigma^2
    \left(
      \frac{d}{qn}
      \vee
      \frac{ds}{q\rho n^2}
    \right).
  \end{equation}
\end{proposition}

Thus each coordinate receives information from \(q\) clients, while each
client spends its privacy information budget over \(s\) observed coordinates.
The derivation is in \cref{app:deriv-balanced-coordinate-mean}.

\subsection{Linear Regression}
\label{subsec:app-linear-regression}

\subsubsection{Homogeneous random design}
\label{subsubsec:app-homogeneous-linear-regression}

To connect mean estimation with regression, consider the random design linear
model
\begin{equation}
  \label{eq:app-homogeneous-linear-regression-model}
  Y_i^{(l)}
  =
  (Z_i^{(l)})^\T\theta+\xi_i^{(l)},
  \qquad
  \xi_i^{(l)}\sim \mca{N}(0,\sigma^2),
\end{equation}
where \(Z_i^{(l)}\) is independent of the noise and
\begin{equation}
  \label{eq:app-homogeneous-linear-regression-design}
  c_0 I_d\preceq \E Z_i^{(l)}(Z_i^{(l)})^\T\preceq c_1 I_d
\end{equation}
uniformly in \(l\), for fixed constants \(0<c_0\le c_1<\infty\).

\begin{proposition}[Homogeneous random design linear regression]
  \label{prop:app-homogeneous-linear-regression}
  Under the model \cref{eq:app-homogeneous-linear-regression-model} and design
  condition \cref{eq:app-homogeneous-linear-regression-design}, we have
  \begin{equation}
    \label{eq:app-homogeneous-linear-regression-rate}
    \inf_{\hat\theta\in \mca{M}_{\bm{\rho}}}
    \sup_{\theta\in [-1,1]^d}
    \E_\theta\norm{\hat\theta-\theta}_2^2
    \gtrsim
    d
    \wedge
    \sigma^2
    \zk{
      \sum_{l=1}^m
      \left(
        \frac{d}{n_l}
        \vee
        \frac{d^2}{\rho_l n_l^2}
      \right)^{-1}
    }^{-1}.
  \end{equation}
\end{proposition}

The rate has the same form as Gaussian mean estimation because the
Fisher information of one observation is uniformly comparable to
\(\sigma^{-2}I_d\) in the well conditioned regime.
The derivation is in \cref{app:deriv-linear-regression}.

\subsubsection{Anisotropic and heterogeneous designs}
\label{subsubsec:app-heterogeneous-linear-regression}

For heterogeneous designs, the matrix form keeps the orientation of the local
Fisher information.
Suppose client \(l\) observes
\begin{equation}
  \label{eq:app-heterogeneous-linear-regression-model}
  Y_i^{(l)}
  =
  (Z_i^{(l)})^\T\theta+\xi_i^{(l)},
  \qquad
  \xi_i^{(l)}\sim \mca{N}(0,\sigma_l^2),
  \qquad
  \E Z_i^{(l)}(Z_i^{(l)})^\T=\Sigma_l .
\end{equation}

\begin{proposition}[Anisotropic and heterogeneous linear regression]
  \label{prop:app-heterogeneous-linear-regression}
  Under the model \cref{eq:app-heterogeneous-linear-regression-model}, we have
  \begin{equation}
    \label{eq:app-heterogeneous-linear-regression-bound}
    \inf_{\hat\theta\in \mca{M}_{\bm{\rho}}}
    \sup_{\theta\in [-1,1]^d}
    \E_\theta\norm{\hat\theta-\theta}_2^2
    \ge
    \inf_{\{B_l\}}
    \Tr
    \left[
      c I_d+
      \sum_{l=1}^m
      n_l\sigma_l^{-2}\Sigma_l^{1/2}B_l\Sigma_l^{1/2}
    \right]^{-1},
  \end{equation}
  where
  \begin{equation}
    \label{eq:app-heterogeneous-linear-regression-feasible}
    0\preceq B_l\preceq P_{\Sigma_l},
    \qquad
    \Tr B_l\le (e^{2\rho_l}-1)n_l .
  \end{equation}
\end{proposition}

Here \(c>0\) is a universal constant and \(P_{\Sigma_l}\) denotes the
orthogonal projection onto \(\operatorname{range}(\Sigma_l)\).
The transcript contributes no information in directions outside
\(\operatorname{span}\{\operatorname{range}(\Sigma_l):1\le l\le m\}\).
This is the directional obstruction that a scalar trace denominator does not
record.
The derivation is in \cref{app:deriv-heterogeneous-linear-regression}.

\subsubsection{Target domain prediction}
\label{subsubsec:app-target-domain-prediction}

The same matrix form also applies when prediction is evaluated only on a target
subspace.
In the heterogeneous linear model
\cref{eq:app-heterogeneous-linear-regression-model}, let \(P_0\) be the projection onto an
\(r\)-dimensional target subspace and consider
\begin{equation}
  \label{eq:app-target-domain-loss}
  L_0(\hat\theta,\theta)
  =
  (\hat\theta-\theta)^\T P_0(\hat\theta-\theta).
\end{equation}
Suppose exactly \(q\) clients are
target relevant with \(n_l=n\), \(\rho_l=\rho\), \(\sigma_l=\sigma\), and
\(\Sigma_l=P_0\), while all other clients satisfy \(P_0\Sigma_lP_0=0\):
\begin{equation}
  \label{eq:app-target-domain-relevance}
  \begin{gathered}
    \#\{l:\Sigma_l=P_0,\ n_l=n,\ \rho_l=\rho,\ \sigma_l=\sigma\}=q,\\
    P_0\Sigma_lP_0=0\quad\text{for all other clients.}
  \end{gathered}
\end{equation}

\begin{proposition}[Target domain prediction]
  \label{prop:app-target-domain-prediction}
  Under the model \cref{eq:app-heterogeneous-linear-regression-model}, target
  loss \cref{eq:app-target-domain-loss}, and relevance condition
  \cref{eq:app-target-domain-relevance}, we have
  \begin{equation}
    \label{eq:app-target-domain-prediction-rate}
    \inf_{\hat\theta\in \mca{M}_{\bm{\rho}}}
    \sup_{\theta\in B_2(0,1)}
    \E_\theta L_0(\hat\theta,\theta)
    \gtrsim
    1
    \wedge
    \sigma^2
    \left(
      \frac{r}{qn}
      \vee
      \frac{r^2}{q\rho n^2}
    \right).
  \end{equation}
\end{proposition}

Only the target relevant clients contribute to this lower bound; information in
orthogonal source directions is irrelevant for \(L_0\).
A trace calculation alone cannot make this distinction, because it summarizes
source clients by total Fisher mass rather than by Fisher mass in the target
subspace.
The matrix bound removes clients whose information is supported in
\(P_0^\perp\) from the denominator of the lower bound for the target risk.
The derivation is in \cref{app:deriv-target-domain-prediction}.

\subsection{Nonparametric Regression}
\label{subsec:app-nonparametric-regression}

The scalar inequality also gives nonparametric rates after projection onto
finite-dimensional Sobolev submodels.
Consider the nonparametric regression model on \([0,1]^d\),
\begin{equation}
  \label{eq:app-nonparametric-regression-model}
  Y_i^{(l)}=f(U_i^{(l)})+\xi_i^{(l)},
  \qquad
  \xi_i^{(l)}\sim \mca{N}(0,\sigma^2),
\end{equation}
where the design density is bounded above and below and
\(f\in H^\alpha(R)\).

\begin{proposition}[Nonparametric regression]
  \label{prop:app-nonparametric-regression}
  Under the model \cref{eq:app-nonparametric-regression-model}, we have
  \begin{equation}
    \label{eq:app-nonparametric-regression-general-rate}
    \inf_{\hat f\in\mca{M}_{\bm{\rho}}}
    \sup_{f\in H^\alpha(R)}
    \E_f\norm{\hat f-f}_{L^2}^2
    \gtrsim
    \sup_{p\ge 1}
    \zk{
      \frac{p^{2\alpha/d}}{R^2}
      +
      \frac{1}{\sigma^2}
      \sum_{l=1}^m
      \left(
        \frac{p}{n_l}
        \vee
        \frac{p^2}{\rho_l n_l^2}
      \right)^{-1}
    }^{-1}.
  \end{equation}
  In the homogeneous case \(n_l=n\) and \(\rho_l=\rho\), this implies
  \begin{equation}
    \label{eq:app-nonparametric-regression-homogeneous-rate}
    \inf_{\hat f\in\mca{M}_{\bm{\rho}}}
    \sup_{f\in H^\alpha(R)}
    \E_f\norm{\hat f-f}_{L^2}^2
    \gtrsim
    R^{\frac{2d}{2\alpha+d}}
    \left(
      \frac{\sigma^2}{mn}
    \right)^{\frac{2\alpha}{2\alpha+d}}
    +
    R^{\frac{2d}{\alpha+d}}
    \left(
      \frac{\sigma^2}{\rho m n^2}
    \right)^{\frac{\alpha}{\alpha+d}} .
  \end{equation}
\end{proposition}

The supremum over \(p\) comes from applying the lower bound to each
\(p\)-dimensional Sobolev submodel and then taking the strongest resulting
lower bound.
Here \(p\) indexes the finite-dimensional reduction: the first term is the
Sobolev truncation contribution, and the sum is the transcript information
available at that resolution.
The first component is the ordinary nonprivate sampling term, while the
second is the privacy limited term.
The derivation is in \cref{app:deriv-nonparametric-regression}.

\subsection{Functional Mean Estimation}
\label{subsec:app-functional-mean}

We finally consider functional mean estimation from discretely observed curves.
On client \(l\), the \(n_l\) private samples are independent curves.
The \(i\)-th curve is observed through
\begin{equation}
  \label{eq:app-functional-mean-model}
  Y_{ij}^{(l)}
  =
  X_i^{(l)}(T_{ij}^{(l)})+\xi_{ij}^{(l)},
  \qquad
  j=1,\ldots,k_l,
\end{equation}
where \(\xi_{ij}^{(l)}\sim\mca{N}(0,1)\) and
\(f=\E X_i^{(l)}\in H^\alpha(R)\), \(\alpha>1/2\), is the mean function.

In the terminology of \cref{def:full-transcript-zcdp}, the sample-level
privacy unit is one whole discretely observed curve
\(\{(T_{ij}^{(l)},Y_{ij}^{(l)}):1\le j\le k_l\}\), not one individual grid
measurement.
Thus adjacent local datasets differ by replacing one curve, and the lower
bounds combine measurement sparsity with variation at the curve level.

We consider two design settings.
In independent design, the measurement locations are part of the private curve
sample and are independently sampled for each curve.
In common design, the grid is public and common across curves.
The lower bounds have the same form as the lower bounds for noninteractive
algorithms in \citet{cai2024_OptimalFederatedFunctional}.

\subsubsection{Independent design}
\label{subsubsec:app-functional-independent}

Suppose the measurement locations are private and independently sampled for
each curve:
\begin{equation}
  \label{eq:app-functional-independent-design}
  T_{ij}^{(l)}\stackrel{\mathrm{ind}}{\sim}\mu_l,
  \qquad
  0<c_0\le \frac{\dd\mu_l}{\dd t}\le c_1<\infty .
\end{equation}
To use the van Trees lower bound,
we consider finite-dimensional submodels
\(f_\theta=\sum_{r=1}^p\theta_r\phi_r\), where \(p\) is the resolution
parameter and \(\phi_r\) is taken from the Fourier basis.
Let \(I_{l,p}^{\mathrm{ind}}(\theta)\) be the Fisher information matrix in the
coefficient parameter from one discretely observed curve from client \(l\).
The information contribution from one curve is limited by both the number of
measurements \(k_l\) and Gaussian variation at the curve level:
\[
  \norm{I_{l,p}^{\mathrm{ind}}(\theta)}_{\mathrm{op}}
  \lesssim k_l\wedge p^{2\alpha+1},
  \qquad
  \Tr I_{l,p}^{\mathrm{ind}}(\theta)
  \lesssim k_lp\wedge p^{2\alpha+2}.
\]
The \(k_l\)-terms come from the fact that only \(k_l\) random measurement
locations are observed for a curve.
The \(p^{2\alpha+1}\)-terms come from the Gaussian variation at the curve
level, which limits how much information one whole curve can carry about the
\(p\) mean coefficients.
Plugging these bounds into the van Trees lower bound gives the following result.
The detailed derivation is given in \cref{app:deriv-functional-independent}.

\begin{proposition}[Functional mean estimation with independent design]
  \label{prop:app-functional-independent}
  Under the model \cref{eq:app-functional-mean-model} and independent design
  condition \cref{eq:app-functional-independent-design}, we have
  \begin{equation}
    \label{eq:app-functional-independent-general-rate}
    \begin{aligned}
      &\inf_{\hat f\in\mca{M}_{\bm{\rho}}}
      \sup_{f\in H^\alpha(R)}
      \E_f\norm{\hat f-f}_{L^2}^2 \\
      &\qquad\gtrsim
      \sup_{p\ge1}
      \left\{
        p^{2\alpha}
        +
        \sum_{l=1}^m
        \left(
          \frac{k_l n_l}{p}
          \wedge
          \frac{\rho_l k_l n_l^2}{p^2}
          \wedge
          p^{2\alpha}n_l
          \wedge
          p^{2\alpha-1}\rho_l n_l^2
        \right)
      \right\}^{-1}.
    \end{aligned}
  \end{equation}
  In the homogeneous case \(n_l=n\), \(k_l=k\), and \(\rho_l=\rho\), this gives
  \begin{equation}
    \label{eq:app-functional-independent-homogeneous-rate}
    \inf_{\hat f\in\mca{M}_{\bm{\rho}}}
    \sup_{f\in H^\alpha(R)}
    \E_f\norm{\hat f-f}_{L^2}^2
    \gtrsim
    \frac{1}{mn}
    +
    (mkn)^{-\frac{2\alpha}{2\alpha+1}}
    +
    (\rho m k n^2)^{-\frac{\alpha}{\alpha+1}}
    +
    \frac{1}{\rho m n^2}.
  \end{equation}
\end{proposition}

In the homogeneous independent-design case, the privacy term follows by
balancing the Sobolev contribution \(p^{2\alpha}\) with
\(\rho m k n^2/p^2\), while the low-frequency privacy bottleneck is obtained
at \(p=1\).
The full optimization over \(p\) is given in
\cref{app:deriv-functional-independent}.

\subsubsection{Common design}
\label{subsubsec:app-functional-common}

Suppose the grid is public and common across curves, with \(k_l=k\):
\begin{equation}
  \label{eq:app-functional-common-design}
  T_{ij}^{(l)}=t_j,
  \qquad j=1,\ldots,k,\quad l=1,\ldots,m .
\end{equation}
We again use finite-dimensional submodels
\(f_\theta=\sum_{r=1}^p\theta_r\phi_r\), now with \(\phi_r\) chosen as
localized bump functions aligned with the public grid.
Let \(I_{p}^{\mathrm{com}}(\theta)\) be the Fisher information matrix in the
coefficient parameter from one whole curve observed on the common grid.
In this case, we have Fisher information control
\[
  \norm{I_p^{\mathrm{com}}(\theta)}_{\mathrm{op}}\lesssim p,\quad \Tr I_p^{\mathrm{com}}(\theta)\lesssim p^2.
\]
Together with the usual discretization obstruction term \(k^{-2\alpha}\), we
have the following lower bound, whose derivation is in
\cref{app:deriv-functional-common}.

\begin{proposition}[Functional mean estimation with common design]
  \label{prop:app-functional-common}
  Under the model \cref{eq:app-functional-mean-model} and common design
  condition \cref{eq:app-functional-common-design}, we have
  \begin{equation}
    \label{eq:app-functional-common-general-rate}
    \inf_{\hat f\in\mca{M}_{\bm{\rho}}}
    \sup_{f\in H^\alpha(R)}
    \E_f\norm{\hat f-f}_{L^2}^2
    \gtrsim
    k^{-2\alpha}
    +
    \sup_{1\le p\le k}
    \left\{
      p^{2\alpha}
      +
      \sum_{l=1}^m
      \left(
        n_l
        \wedge
        \frac{\rho_l n_l^2}{p}
      \right)
    \right\}^{-1}.
  \end{equation}
  In the homogeneous case \(n_l=n\) and \(\rho_l=\rho\),
  \begin{equation}
    \label{eq:app-functional-common-homogeneous-rate}
    \inf_{\hat f\in\mca{M}_{\bm{\rho}}}
    \sup_{f\in H^\alpha(R)}
    \E_f\norm{\hat f-f}_{L^2}^2
    \gtrsim
    \frac{1}{mn}
    +
    k^{-2\alpha}
    +
    (\rho m n^2)^{-\frac{2\alpha}{2\alpha+1}},
  \end{equation}
\end{proposition}

For common design, the privacy exponent comes from balancing
\(p^{2\alpha}\) with \(\rho m n^2/p\), subject to \(p\le k\).
When this resolution exceeds the public grid, the discretization term
\(k^{-2\alpha}\) is active.
The full optimization over \(p\) is given in \cref{app:deriv-functional-common}.
 \section{Conclusion}
\label{sec:conclusion}

This paper developed a van Trees lower bound framework for federated learning
under zCDP at the sample level with complete public transcripts.
The key feature of the framework is that the transcript may be fully
interactive: clients may communicate over many adaptive rounds, the server's
later queries may depend on earlier public messages, and local samples may be
reused across rounds.
The privacy constraint is imposed only on the complete ordered transcript, and
each client is charged by its total zCDP budget.

The main technical ingredient is a Fisher information contraction inequality
for the complete transcript under privacy.
In homogeneous models, this gives a scalar trace bound and a federated van
Trees inequality that reduce lower bound calculations to the Fisher information
of one observation and the clientwise privacy budgets.
For heterogeneous models, the matrix extension keeps the directional Fisher
geometry of the local experiments.
It decomposes the transcript information into clientwise positive semidefinite
contributions, each constrained by the client's raw Fisher information and a
whitened privacy information budget.
This matrix form is essential when different clients observe different
subspaces or when the loss depends only on target directions.

The resulting inequalities separate the model specific Fisher calculation from
the privacy and interaction part of the lower bound argument.
In the applications considered here, this separation gives minimax lower bounds
for the full class of interactive protocols with public transcripts in mean
estimation, linear regression, nonparametric regression, and functional mean
estimation.
Together with existing upper bounds for the corresponding problems, these lower
bounds identify the minimax rates.
Consequently, the simpler one round or otherwise restricted protocols attaining
these rates are rate optimal even after the algorithm class is enlarged to
allow arbitrary interaction over multiple rounds and repeated sample reuse.
The same proof architecture also extends to approximate DP by replacing the
zCDP information contraction with a truncation-based contraction; the appendix
states a representative homogeneous scalar form and the same replacement applies
to the matrix framework.

Several directions remain open.
One is to develop similarly general lower bounds for other privacy and trust
models, such as secure aggregation, shuffle privacy, or protocols with trusted
private server state.
Another is to combine complete transcript privacy contraction with additional
algorithmic constraints, such as communication limits or computational
restrictions.
Finally, the matrix formulation suggests further applications in heterogeneous
and transfer settings, where the relevant difficulty is determined by the
directions covered by the clients rather than by total Fisher information
alone.

\clearpage
\phantomsection

\clearpage

\appendix
\section{Proof}
\label{sec:proof}

The following sufficient conditions are standard for Bayesian Cram\'er--Rao or
van Trees inequalities; see
\citet{gill1995_ApplicationsVan,gassiat2024_VanTrees} and
Lemma C.1 of \citet{cai2026_MinimaxAdaptive}.
We state them for the heterogeneous product model.
For each \(\theta\), the raw data law is
\[
  P_\theta^X
  =
  \bigotimes_{l=1}^m\bigotimes_{i=1}^{n_l}P_{\theta,li},
  \qquad
  X_i^{(l)}\sim P_{\theta,li}.
\]
Let \(p_{\theta,li}\) be the density of \(P_{\theta,li}\) and define the local
score and information matrix
\[
  s_{li}(X_i^{(l)};\theta)
  =
  \nabla_\theta\log p_{\theta,li}(X_i^{(l)}),
  \qquad
  I_{li}(\theta)
  =
  \E_\theta[
    s_{li}(X_i^{(l)};\theta)s_{li}(X_i^{(l)};\theta)^\T
  ].
\]
For the client and full data scores, write
\[
  S_l(\theta)=\sum_{i=1}^{n_l}s_{li}(X_i^{(l)};\theta),
  \qquad
  S(\theta)=\sum_{l=1}^m S_l(\theta).
\]
Regularity of the raw data gives
\begin{equation}
  \label{eq:raw-score-regularity}
  \begin{aligned}
    \E_\theta s_{li}(X_i^{(l)};\theta)&=0,
    \qquad l=1,\ldots,m,\quad i=1,\ldots,n_l,\\
    I_{li}(\theta)
    &=
    \operatorname{Cov}_\theta(
      s_{li}(X_i^{(l)};\theta)
    ),\\
    J_l(\theta)
    &:=
    \E_\theta[S_l(\theta)S_l(\theta)^\T]
    =
    \sum_{i=1}^{n_l}I_{li}(\theta).
  \end{aligned}
\end{equation}
Consequently, \(\E_\theta S_l(\theta)=0\) for each \(l\).
For a possibly randomized statistic \(T\) generated from the raw data, let
\(p_{\theta,T}\) be the density of its induced law, let \(I_T(\theta)\) be its
Fisher information matrix, and let \(S_X(\theta)\) be the score of the raw data
entering \(T\).
The corresponding score projection identities are
\begin{equation}
  \label{eq:transcript-score-projection}
  \begin{aligned}
    \nabla_\theta\log p_{\theta,T}(T)
    &=
    \E_\theta[S_X(\theta)\mid T],\\
    I_T(\theta)
    &=
    \E_\theta\left[
      \E_\theta[S_X(\theta)\mid T]
      \E_\theta[S_X(\theta)\mid T]^\T
    \right].
  \end{aligned}
\end{equation}
For the complete public transcript \(G\), write \(I_G(\theta)\) for the Fisher
information matrix of its induced law.

\begin{assumption}[Regularity for heterogeneous van Trees]
  \label{ass:van-trees-regularity}
  For the heterogeneous product model above, we use the following standard
  sufficient conditions for the van Trees inequality.
  They are not intended to be minimal.
  \begin{enumerate}[label=(\roman*)]
    \item For each \(l\) and \(i\), the density \(p_{\theta,li}\) is defined
    with respect to a \(\theta\)-independent dominating measure and is
    continuously differentiable in \(\theta\).
    \item For each \(l\) and \(i\), the local likelihood satisfies the
    conditions for differentiating under the integral sign that yield the local
    identities in \cref{eq:raw-score-regularity}.
    \item The client Fisher information matrices \(J_l(\theta)\) are well
    defined and
    \[
      \int_\Theta \norm{J_l(\theta)}_{\mathrm{op}}\dd\pi(\theta)<\infty,
      \qquad l=1,\ldots,m.
    \]
    \item The parameter space \(\Theta\) is a compact subset of \(\R^p\) with
    piecewise smooth boundary.
    \item The prior distribution \(\pi\) has a continuously differentiable
    density on \(\R^p\) and vanishes on the boundary of \(\Theta\).
    \item The prior Fisher information matrix
    \[
      J_\pi
      =
      \int_\Theta
      \nabla_\theta\log\pi(\theta)
      \nabla_\theta\log\pi(\theta)^\T
      \pi(\theta)\dd\theta
    \]
    is well defined and finite.
    \item The data, private history, and transcript spaces are standard Borel,
    so regular conditional distributions for randomized transcripts exist.
    The induced transcript laws are dominated and differentiable in \(\theta\)
    and satisfy
    \cref{eq:transcript-score-projection}.
  \end{enumerate}
\end{assumption}

\subsection{Single client information contraction}
\label{subsec:proof-single-client-information-contraction}

The scalar trace argument starts from the following zCDP information
contraction for one client, due to \citet{cai2026_MinimaxAdaptive}.

\begin{lemma}[Single client zCDP information contraction]
  \label{lem:single-client-zcdp-information}
  Let \(X_1,\ldots,X_n\) be i.i.d.\ samples from \(P_\theta\), with score
  \(s(X;\theta)\) and one-sample Fisher information \(I_x(\theta)\).
  Let \(T\) be a statistic generated from \(X=(X_1,\ldots,X_n)\) by a
  \(\theta\)-independent \(\rho\)-zCDP mechanism at the sample level.
  Write
  \(S_X(\theta)=\sum_{i=1}^n s(X_i;\theta)\). Then
  \[
    \E_\theta\norm{\E_\theta[S_X(\theta)\mid T]}_2^2
    =
    \Tr I_T(\theta)
    \le
    (e^{2\rho}-1)n^2\norm{I_x(\theta)}_{\mathrm{op}}.
  \]
\end{lemma}

The matrix argument uses an affine invariant strengthening.
Its proof adapts the zCDP information contraction argument of
Lemmas C.2 and C.3 of \citet{cai2026_MinimaxAdaptive} after whitening the
score.

\begin{lemma}[Single client matrix zCDP information contraction]
  \label{lem:single-client-matrix-zcdp-information}
  Let \(X_1,\ldots,X_n\) be independent observations from models
  \(P_{\theta,i}\) that are not necessarily identical.
  Let \(s_i(X_i;\theta)\) and \(I_i(\theta)\) be the score and Fisher
  information matrix of \(X_i\).
  Let
  \[
    S(\theta)=\sum_{i=1}^n s_i(X_i;\theta),
    \qquad
    J(\theta)=\E_\theta[S(\theta)S(\theta)^\T]
    =
    \sum_{i=1}^n I_i(\theta).
  \]
  Let \(T=M(X)\) be generated by a \(\theta\)-independent \(\rho\)-zCDP mechanism at the sample level, and define
  \[
    U(T;\theta)=\E_\theta[S(\theta)\mid T],
    \qquad
    A(\theta)=\E_\theta[U(T;\theta)U(T;\theta)^\T].
  \]
  Then
  \[
    0\preceq A(\theta)\preceq J(\theta),
    \qquad
    \Tr\zk{J(\theta)^\dagger A(\theta)}
    \le
    (e^{2\rho}-1)n.
  \]
\end{lemma}

\begin{proof}
  Fix \(\theta\), let all expectations in this proof be under \(P_\theta\), and
  write \(S,J,U,A\) for the corresponding quantities.
  We first assume \(J\succ0\).
  The law of total covariance gives
  \[
    J
    =
    \E[\operatorname{Cov}(S\mid T)]
    +
    \operatorname{Cov}(\E[S\mid T]),
  \]
  and since \(\E S=0\), the second term is \(A\). Hence \(0\preceq A\preceq J\).

  Set \(W(T)=J^{-1}U(T)\) and
  \[
    a=\Tr(J^{-1}A)
    =
    \E[U^\T J^{-1}U]
    =
    \E\ang{W(T),U(T)}.
  \]
  Since \(U(T)=\E[S\mid T]\),
  \[
    a
    =
    \E\ang{W(T),S}
    =
    \sum_{i=1}^n \E\ang{W(T),s_i(X_i)}.
  \]

  Fix \(i\). Let \(X_i'\) be an independent copy of \(X_i\), define
  \[
    X^{(i)}=(X_1,\ldots,X_{i-1},X_i',X_{i+1},\ldots,X_n),
    \qquad
    T^{(i)}=M(X^{(i)}),
  \]
  and write \(W^{(i)}=W(T^{(i)})\).
  Since \(X^{(i)}\) is independent of \(X_i\) and \(\E s_i(X_i)=0\),
  \[
    \E\ang{W^{(i)},s_i(X_i)}=0.
  \]
  Therefore
  \[
    \E\ang{W(T),s_i(X_i)}
    =
    \E\ang{W(T)-W^{(i)},s_i(X_i)}.
  \]

  The datasets \(X\) and \(X^{(i)}\) are adjacent, so zCDP gives, conditional on \((X,X^{(i)})\),
  \[
    D_2(M(X)\|M(X^{(i)}))\le 2\rho.
  \]
  Hence, using the relation between order-2 R\'enyi divergence and \(\chi^2\) divergence \citep{erven2014_RenyiDivergence},
  \[
    \chi^2(T \|T^{(i)}\mid X,X^{(i)})
    =
    \chi^2(M(X)\|M(X^{(i)}))\le e^{2\rho}-1\eqqcolon\kappa.
  \]
  We use the chi-square mean difference inequality in Lemma F.1 of
  \citet{cai2026_MinimaxAdaptive},
  \[
    \abs{\E_P f-\E_Q f}^2
    \le
    \chi^2(P\|Q)\E_Q[f^2],
  \]
  for \(P\ll Q\).
  Applying it conditionally on \((X,X^{(i)})\), with
  \(P=\mca L(T\mid X,X^{(i)})\), \(Q=\mca L(T^{(i)}\mid X,X^{(i)})\), and
  \(f(t)=\ang{W(t),s_i(X_i)}\), gives
  \[
    \left(
      \E[f(T)\mid X,X^{(i)}]
      -
      \E[f(T^{(i)})\mid X,X^{(i)}]
    \right)^2
    \le
    \kappa
    \E[f(T^{(i)})^2\mid X,X^{(i)}].
  \]
  Averaging and applying Jensen's inequality yields
  \[
    \left|\E\ang{W(T),s_i(X_i)}\right|^2
    \le
    \kappa\,
    \E\ang{W(T^{(i)}),s_i(X_i)}^2.
  \]
  Since \(T^{(i)}\) is independent of \(X_i\) and has the same marginal distribution as \(T\),
  \[
    \E\ang{W(T^{(i)}),s_i(X_i)}^2
    =
    \Tr\left(\E[W(T)W(T)^\T]I_i\right)
    =
    \Tr(J^{-1}AJ^{-1}I_i).
  \]
  Denote the last quantity by \(b_i\). Then
  by the triangle inequality and Cauchy--Schwarz,
  \[
    a
    \le
    \sqrt{\kappa}\sum_{i=1}^n\sqrt{b_i}
    \le
    \sqrt{\kappa}\sqrt{n\sum_{i=1}^n b_i}.
  \]
  But
  \[
    \sum_{i=1}^n b_i
    =
    \Tr\left(J^{-1}AJ^{-1}\sum_{i=1}^n I_i\right)
    =
    \Tr(J^{-1}A)
    =
    a.
  \]
  Thus \(a\le \sqrt{\kappa n a}\). If \(a>0\), dividing by \(\sqrt a\) gives
  \(a\le \kappa n\); if \(a=0\), the same conclusion is trivial.

  If \(J\) is singular, let \(R=\operatorname{range}(J)\).
  For \(v\in R^\perp=\operatorname{ker}(J)\),
  \[
    0=v^\T Jv=\E(v^\T S)^2,
  \]
  so \(v^\T S=0\) almost surely.
  Hence \(v^\T U=\E[v^\T S\mid T]=0\) almost surely, and
  \(v^\T Av=\E(v^\T U)^2=0\).
  Since \(A\succeq0\), this implies \(Av=0\) for all \(v\in R^\perp\), or
  equivalently \(\operatorname{range}(A)\subseteq R\).
  If \(R=\{0\}\), then \(A=0\) and the claim is trivial.
  Otherwise, choose an orthonormal matrix \(Q\) whose columns span \(R\) and
  apply the preceding positive definite argument to
  \(\widetilde S=Q^\T S\), \(\widetilde U=Q^\T U\),
  \(\widetilde J=Q^\T JQ\), and \(\widetilde A=Q^\T AQ\).
  Then \(\widetilde J\succ0\), and
  \[
    \Tr(\widetilde J^{-1}\widetilde A)\le \kappa n.
  \]
  Because \(\operatorname{range}(A)\subseteq R\), the last trace equals
  \(\Tr(J^\dagger A)\).
\end{proof}

\subsection{Conditional independence}
\label{subsec:proof-posterior-factorization}

This subsection supplies the factorization step that allows the transcript Fisher information to be charged client by client.
Using the complete public transcript structure,
once a transcript value \(g\) is fixed, all public adaptive choices and
post-processing records are fixed, and each unreleased private history enters
only through the likelihood factor of its own client.
This gives posterior factorization of the local datasets given \(G=g\), and
then the matrix score decomposition follows by projecting the raw client scores
onto the sigma-field generated by \(G\).

\begin{lemma}[Posterior factorization]
  \label{lem:posterior-factorization}
  Under the public transcript protocol and independent client sampling model,
  for \(P_{\theta,G}\)-almost every transcript value \(g\),
  \[
    \mca L_\theta(X^{(1)},\ldots,X^{(m)}\mid G=g)
    =
    \bigotimes_{l=1}^m
    \mca L_\theta(X^{(l)}\mid G=g).
  \]
\end{lemma}

\begin{proof}
  Let \(\theta\) be fixed.
  To expose the core idea, we first prove the result in a dominated setting
  where the local randomizations, local private history updates, initial
  randomizations, and public post-processing kernels admit densities.
  The general argument with probability kernels is given afterward.
  Let \(\nu_0(g_0)\) be the density of the initial public record \(G_0\), and
  write client \(l\)'s private history path as
  \(h_{0:T}^{(l)}=(h_0^{(l)},\ldots,h_T^{(l)})\).
  Let
  \[
    \lambda_l(h_0^{(l)}\mid x^{(l)},g_0)
  \]
  be the conditional density of the initial private history \(H_0^{(l)}\) given
  local data \(x^{(l)}\) and the initial public record \(g_0\), and let
  \[
    q_{t,l}(g_t^{(l)},h_t^{(l)}
      \mid x^{(l)},g_{<t},h_{t-1}^{(l)}),
  \]
  be the conditional density of the pair \((G_t^{(l)},H_t^{(l)})\) given local
  data \(x^{(l)}\), public history \(g_{<t}\), and previous private history
  \(h_{t-1}^{(l)}\).
  Finally, let
  \[
    \alpha_t(a_t\mid g_{<t},g_t^{(1)},\ldots,g_t^{(m)})
  \]
  be the density of the public post-processing record \(A_t\) at the end of the
  round given
  the already public information in that round.
  Deterministic messages, deterministic private history updates, and
  deterministic public post-processing are interpreted as point masses in this
  notation.

  For a realized complete transcript
  \[
    g=
    \bigl(
      g_0,
      (g_1^{(1)},\ldots,g_1^{(m)},a_1),
      \ldots,
      (g_T^{(1)},\ldots,g_T^{(m)},a_T)
    \bigr),
  \]
  the conditional density of \(G=g\) given all datasets is
  \[
    \begin{aligned}
      K(g\mid x^{(1)},\ldots,x^{(m)})
      &=
      P_{\mathrm{pub}}(g)
      \prod_{l=1}^m
      \int
      \lambda_l(h_0^{(l)}\mid x^{(l)},g_0) \\
      &\qquad\qquad{}\times
      \prod_{t=1}^T
      q_{t,l}(g_t^{(l)},h_t^{(l)}
        \mid x^{(l)},g_{<t},h_{t-1}^{(l)})
      \dd h_{0:T}^{(l)} ,
    \end{aligned}
  \]
  where the public factor is
  \[
    P_{\mathrm{pub}}(g)
    =
    \nu_0(g_0)
    \prod_{t=1}^T
    \alpha_t(a_t\mid g_{<t},g_t^{(1)},\ldots,g_t^{(m)}).
  \]
  Since \(g\) is fixed, all adaptive public histories and all \(A_t\)-factors in
  this display are constants.
  The private histories remain inside client \(l\)'s own integral and are not
  shared across clients.
  Define
  \[
    L_l(g;x^{(l)})
    =
    \int
    \lambda_l(h_0^{(l)}\mid x^{(l)},g_0)
    \prod_{t=1}^T
    q_{t,l}(g_t^{(l)},h_t^{(l)}
      \mid x^{(l)},g_{<t},h_{t-1}^{(l)})
    \dd h_{0:T}^{(l)}.
  \]
  The public factor \(P_{\mathrm{pub}}(g)\) is data independent after the transcript value is
  fixed, so it may be absorbed into one local factor, for instance by setting
  \(K_1(g\mid x^{(1)})=P_{\mathrm{pub}}(g)L_1(g;x^{(1)})\) and
  \(K_l(g\mid x^{(l)})=L_l(g;x^{(l)})\) for \(l\ge2\).
  Hence
  \[
    K(g\mid x^{(1)},\ldots,x^{(m)})
    =
    \prod_{l=1}^m K_l(g\mid x^{(l)}).
  \]

  By independence across clients,
  \[
    p_\theta(x^{(1)},\ldots,x^{(m)})
    =
    \prod_{l=1}^m p_{\theta,l}(x^{(l)}),
  \]
  where \(p_{\theta,l}\) is the local joint density of client \(l\)'s data.
  Bayes' formula gives
  \[
    p_\theta(x^{(1)},\ldots,x^{(m)}\mid g)
    =
    \prod_{l=1}^m
    \frac{
      p_{\theta,l}(x^{(l)})K_l(g\mid x^{(l)})
    }{
      Z_l(g)
    },
  \]
  with \(Z_l(g)=\int p_{\theta,l}(u)K_l(g\mid u)\dd u\).
  This proves posterior independence of the local datasets given \(G\).

  We now give the corresponding proof using probability kernels, without
  assuming densities.
  Assume the data, private history, and transcript spaces are standard Borel,
  so regular conditional distributions exist.
  Fix \(\theta\).
  The protocol is specified by Markov kernels: the initial public kernel
  \(\nu_0(\dd g_0)\), the initial local private history kernels
  \(\Lambda_l(\dd h_0^{(l)}\mid x^{(l)},g_0)\), the local round kernels
  \[
    Q_{t,l}(\dd g_t^{(l)},\dd h_t^{(l)}
      \mid x^{(l)},g_{<t},h_{t-1}^{(l)}),
  \]
  and the public post-processing kernels
  \[
    \mathsf A_t(\dd a_t\mid g_{<t},g_t^{(1)},\ldots,g_t^{(m)}).
  \]
  Deterministic messages and deterministic post-processing are simply
  degenerate kernels.

  We use the following elementary disintegration fact.
  If, conditional on a public variable \(C\), the pairs
  \((Y_l,M_l)\), \(l=1,\ldots,m\), are independent with product conditional
  law \(\otimes_l R_l(\dd y_l,\dd m_l\mid C)\), then, for almost every
  realized \(m=(m_1,\ldots,m_m)\),
  \[
    \mca L(Y_1,\ldots,Y_m\mid C,M=m)
    =
    \bigotimes_{l=1}^m
    \mca L(Y_l\mid C,M_l=m_l).
  \]
  This is just disintegration of a product kernel with respect to its
  \(M\)-marginal.

  We prove by induction on the public rounds that the local private states are
  conditionally independent given the public transcript available at that time:
  \[
    \mca L_\theta\bigl(
      (X^{(1)},H_t^{(1)}),\ldots,(X^{(m)},H_t^{(m)})
      \mid G_{\le t}
    \bigr)
    =
    \bigotimes_{l=1}^m
    \mca L_\theta(X^{(l)},H_t^{(l)}\mid G_{\le t})
  \]
  for \(P_{\theta,G_{\le t}}\)-almost every \(G_{\le t}\).
  At \(t=0\), the local datasets are independent and, conditional on
  \((X^{(l)},G_0)\), the initial private histories are generated by separate
  local kernels, so the claim holds conditional on \(G_0\).
  Suppose it holds before round \(t\), conditional on \(G_{<t}\).
  Given \(G_{<t}\), each client applies its own kernel \(Q_{t,l}\) to only
  \((X^{(l)},H_{t-1}^{(l)})\) and the public history.
  Therefore the enlarged local blocks
  \[
    \bigl(X^{(l)},H_t^{(l)},G_t^{(l)}\bigr),
    \qquad l=1,\ldots,m,
  \]
  are conditionally independent given \(G_{<t}\).
  Disintegrating this product kernel with respect to the observed public
  messages \((G_t^{(1)},\ldots,G_t^{(m)})\) shows that
  \((X^{(l)},H_t^{(l)})\), \(l=1,\ldots,m\), remain conditionally independent
  given \(G_{<t}\) and those messages.
  The additional public record \(A_t\) is then drawn from a kernel depending
  only on already public quantities, so conditioning on \(A_t\) does not couple
  the local private blocks.
  Hence the induction step gives conditional independence given \(G_{\le t}\).

  Taking \(t=T\) and then marginalizing out the private histories gives
  \[
    \mca L_\theta(X^{(1)},\ldots,X^{(m)}\mid G=g)
    =
    \bigotimes_{l=1}^m
    \mca L_\theta(X^{(l)}\mid G=g)
  \]
  for \(P_{\theta,G}\)-almost every \(g\), which is the desired factorization.
\end{proof}

\begin{corollary}[Matrix score decomposition]
  \label{cor:matrix-score-decomposition}
  Under the conditions of \cref{lem:posterior-factorization}, if
  \[
    U_l(G;\theta)=\E_\theta[S_l(\theta)\mid G],
    \qquad
    A_l(\theta)=\E_\theta[U_l(G;\theta)U_l(G;\theta)^\T],
  \]
  then
  \[
    I_G(\theta)=\sum_{l=1}^m A_l(\theta).
  \]
\end{corollary}

\begin{proof}
  The score projection identity gives
  \[
    \nabla_\theta\log p_{\theta,G}(G)
    =
    \E_\theta[S(\theta)\mid G]
    =
    \sum_{l=1}^m U_l(G;\theta).
  \]
  Hence
  \[
    I_G(\theta)
    =
    \sum_l \E_\theta[U_lU_l^\T]
    +
    \sum_{l\ne k}\E_\theta[U_lU_k^\T].
  \]
  For \(l\ne k\), \cref{lem:posterior-factorization} gives
  \[
    \E_\theta[S_l(\theta)S_k(\theta)^\T\mid G]
    =
    \E_\theta[S_l(\theta)\mid G]\E_\theta[S_k(\theta)\mid G]^\T
    =
    U_lU_k^\T.
  \]
  Taking expectations,
  \[
    \E_\theta[U_lU_k^\T]
    =
    \E_\theta[S_l(\theta)S_k(\theta)^\T]
    =
    0,
  \]
  because the local datasets are unconditionally independent and the local scores have mean zero.
  Therefore \(I_G(\theta)=\sum_l A_l(\theta)\).
\end{proof}

\subsection{Proof of \texorpdfstring{\cref{thm:general-fed-dp-vantrees}}{Theorem 3.1}}
\label{subsec:proof-general-fed-dp-vantrees}

\begin{lemma}[Transcript information contraction]
  \label{lem:transcript-information-contraction}
  Let \(G\) be the complete public transcript of a \(\bm\rho\)-zCDP federated protocol under the common iid model in \cref{subsec:trace-fed-vantrees}.
  Then
  \[
    \Tr I_G(\theta)
    \le
    \sum_{l=1}^m
    \left[
      (e^{2\rho_l}-1)n_l^2\norm{I_x(\theta)}_{\mathrm{op}}
      \wedge
      n_l\Tr I_x(\theta)
    \right].
  \]
\end{lemma}

\begin{proof}
  By \cref{cor:matrix-score-decomposition},
  \[
    \Tr I_G(\theta)
    =
    \sum_{l=1}^m
    \E_\theta\norm{\E_\theta[S_l(\theta)\mid G]}_2^2.
  \]
  The ordinary Fisher bound follows from \(L^2\) contraction of conditional expectation:
  \[
    \E_\theta\norm{\E_\theta[S_l(\theta)\mid G]}_2^2
    \le
    \E_\theta\norm{S_l(\theta)}_2^2
    =
    n_l\Tr I_x(\theta).
  \]

  For the privacy bound, set
  \[
    \widetilde U_l(G,X^{(-l)})
    =
    \E_\theta[S_l(\theta)\mid G,X^{(-l)}].
  \]
  Since conditioning on \((G,X^{(-l)})\) gives a finer sigma-field than conditioning on \(G\),
  Jensen's inequality gives
  \[
    \E_\theta\norm{\E_\theta[S_l(\theta)\mid G]}_2^2
    \le
    \E_\theta\norm{\widetilde U_l(G,X^{(-l)})}_2^2.
  \]
  Conditional on \(X^{(-l)}=x^{(-l)}\), the complete transcript \(G\) is a
  \(\rho_l\)-zCDP statistic of the \(n_l\) i.i.d.\ local samples \(X^{(l)}\).
  Therefore \cref{lem:single-client-zcdp-information}, applied conditionally and then averaged over \(X^{(-l)}\), gives
  \[
    \E_\theta\norm{\E_\theta[S_l(\theta)\mid G]}_2^2
    \le
    (e^{2\rho_l}-1)n_l^2\norm{I_x(\theta)}_{\mathrm{op}}.
  \]
  Combining the two clientwise bounds and summing over \(l\) proves the lemma.
\end{proof}

\begin{proof}[Proof of \cref{thm:general-fed-dp-vantrees}]
  The classical multivariate van Trees inequality
  \citep{gill1995_ApplicationsVan,gassiat2024_VanTrees}, in the form recalled
  in Lemma C.1 of \citet{cai2026_MinimaxAdaptive}, applied to the statistic \(G\)
  gives
  \[
    \E_\pi\E_\theta\norm{\hat\theta-\theta}_2^2
    \ge
    \frac{p^2}{
      \int_\Theta\Tr I_G(\theta)\dd\pi(\theta)
      +
      \Tr J_\pi
    }.
  \]
  Substituting the information contraction bound from
  \cref{lem:transcript-information-contraction} proves the displayed bound.
\end{proof}

\subsection{Proof of \texorpdfstring{\cref{thm:heterogeneous-matrix-information-contraction}}{Matrix transcript contraction} and matrix consequences}
\label{subsec:proof-heterogeneous-matrix-information-contraction}

\begin{proof}[Proof of \cref{thm:heterogeneous-matrix-information-contraction}]
  Let
  \[
    U_l(G;\theta)=\E_\theta[S_l(\theta)\mid G],
    \qquad
    A_l(\theta)=\E_\theta[U_l(G;\theta)U_l(G;\theta)^\T].
  \]
  By \cref{cor:matrix-score-decomposition},
  \[
    I_G(\theta)=\sum_{l=1}^m A_l(\theta).
  \]
  The ordinary matrix contraction follows from the law of total covariance:
  \[
    A_l(\theta)
    =
    \operatorname{Cov}_\theta(\E_\theta[S_l(\theta)\mid G])
    \preceq
    \operatorname{Cov}_\theta(S_l(\theta))
    =
    J_l(\theta).
  \]

  It remains to prove the whitened trace bound.
  Define the finer posterior score
  \[
    \widetilde U_l(G,X^{(-l)};\theta)
    =
    \E_\theta[S_l(\theta)\mid G,X^{(-l)}],
  \]
  and
  \[
    \widetilde A_l(\theta)
    =
    \E_\theta[\widetilde U_l(G,X^{(-l)};\theta)
    \widetilde U_l(G,X^{(-l)};\theta)^\T].
  \]
  Since \(U_l(G;\theta)=\E_\theta[\widetilde U_l(G,X^{(-l)};\theta)\mid G]\), Jensen's inequality in Loewner order gives
  \[
    A_l(\theta)\preceq \widetilde A_l(\theta).
  \]
  Conditional on \(X^{(-l)}=x^{(-l)}\), the complete transcript \(G\) is a \(\rho_l\)-zCDP mechanism at the sample level for the local dataset \(X^{(l)}\).
  The local score covariance is \(J_l(\theta)\), which does not depend on \(x^{(-l)}\).
  Applying \cref{lem:single-client-matrix-zcdp-information} conditionally gives
  \[
    \Tr\left(
      J_l(\theta)^\dagger
      \E_\theta[\widetilde U_l\widetilde U_l^\T\mid X^{(-l)}=x^{(-l)}]
    \right)
    \le
    (e^{2\rho_l}-1)n_l.
  \]
  Averaging over \(X^{(-l)}\) and using \(A_l\preceq\widetilde A_l\) gives
  \[
    \Tr\zk{J_l(\theta)^\dagger A_l(\theta)}
    \le
    \Tr\zk{J_l(\theta)^\dagger \widetilde A_l(\theta)}
    \le
    (e^{2\rho_l}-1)n_l.
  \]

  Finally, since \(0\preceq A_l\preceq J_l\), the range of \(A_l\) is contained in the range of \(J_l\).
  Setting
  \[
    B_l(\theta)=(J_l(\theta)^\dagger)^{1/2}A_l(\theta)(J_l(\theta)^\dagger)^{1/2}
  \]
  yields
  \[
    A_l(\theta)=J_l(\theta)^{1/2}B_l(\theta)J_l(\theta)^{1/2},
    \qquad
    0\preceq B_l(\theta)\preceq P_{J_l(\theta)},
  \]
  and \(\Tr B_l(\theta)=\Tr[J_l(\theta)^\dagger A_l(\theta)]\).
\end{proof}

\begin{proof}[Proof of \cref{thm:matrix-fed-dp-vantrees}]
  The matrix form of the van Trees inequality
  \citep{gill1995_ApplicationsVan,gassiat2024_VanTrees}
  gives
  \[
    \E_\pi\E_\theta
    [(\hat\theta-\theta)(\hat\theta-\theta)^\T]
    \succeq
    \left(
      J_\pi+\int_\Theta I_G(\theta)\dd\pi(\theta)
    \right)^{-1},
  \]
  with the usual regularization if the matrix is singular.
  Now fix an admissible protocol.
  By \cref{thm:heterogeneous-matrix-information-contraction}, this protocol
  induces positive semidefinite matrices \(A_l(\theta)\) and hence feasible
  matrices \(B_l(\theta)\) satisfying
  \[
    I_G(\theta)
    =
    \sum_{l=1}^m
    J_l(\theta)^{1/2}B_l(\theta)J_l(\theta)^{1/2}.
  \]
  Substituting this particular feasible family into the matrix van Trees bound
  gives a lower bound for the fixed protocol.
  The displayed theorem then follows by weakening this protocol specific bound:
  the family induced by the protocol is one admissible measurable choice in the
  feasible class, so its trace inverse value is at least the infimum over all
  feasible measurable choices \(B_l(\theta)\).
\end{proof}

\begin{proof}[Proof of \cref{cor:balanced-coordinate-coverage}]
  Fix any feasible family \(B_l\) in
  \cref{thm:matrix-fed-dp-vantrees}, and set
  \[
    K
    =
    J_\pi+\sum_{l=1}^m
    n\sigma^{-2}P_lB_lP_l .
  \]
  Write
  \[
    b_{lj}=e_j^\T B_le_j,
    \qquad
    a_j=\sum_{l:\,j\in \mca S_l} b_{lj},
    \qquad j=1,\ldots,p,
  \]
  where \(b_{lj}=0\) if \(j\notin \mca S_l\).
  For each client, \(\sum_{j\in \mca S_l}b_{lj}\le r\wedge \kappa n\).
  Therefore
  \[
    \sum_{j=1}^p a_j
    =
    \sum_{l=1}^m\sum_{j\in \mca S_l} b_{lj}
    \le
    m(r\wedge \kappa n)
    =
    pq\left(1\wedge\frac{\kappa n}{r}\right).
  \]
  Also \(K_{jj}\le \lambda_\pi+\sigma^{-2}na_j\).
  By the Schur complement inequality,
  \((K^{-1})_{jj}\ge 1/K_{jj}\), with the same conclusion obtained by
  \(\epsilon I\)-regularization if needed.
  Hence
  \[
    \Tr K^{-1}
    \ge
    \sum_{j=1}^p
    \frac{1}{\lambda_\pi+\sigma^{-2}n a_j}.
  \]
  By convexity of \(x\mapsto(\lambda_\pi+\sigma^{-2}nx)^{-1}\),
  \[
    \sum_{j=1}^p
    \frac{1}{\lambda_\pi+\sigma^{-2}n a_j}
    \ge
    \frac{p}{
      \lambda_\pi
      +
      \sigma^{-2}qn
      \left(1\wedge\frac{\kappa n}{r}\right)
    }.
  \]
  Taking the infimum over feasible \(b_{lj}\) proves the first display.
  The final display follows when \(\lambda_\pi\) is absorbed into the
  transcript information term.
\end{proof}
 \clearpage
\section{Extension to approximate DP}
\label{app:extension-approximate-dp}

\begin{definition}[Full-transcript approximate DP]
  \label{def:full-transcript-approx-dp}
  An interactive federated protocol with complete public transcript \(G\) is
  \((\bm\epsilon,\bm\delta)\)-DP, with
  \(\bm\epsilon=(\epsilon_1,\ldots,\epsilon_m)\) and
  \(\bm\delta=(\delta_1,\ldots,\delta_m)\), if for every client
  \(l\in[m]\), every fixed value of \(X^{(-l)}\), every adjacent pair
  \(X^{(l)},\widetilde X^{(l)}\), and every measurable transcript event \(B\),
  \[
    \Pr\dk{
      G(X^{(l)},X^{(-l)})\in B
    }
    \le
    e^{\epsilon_l}
    \Pr\dk{
      G(\widetilde X^{(l)},X^{(-l)})\in B
    }
    +
    \delta_l .
  \]
  The adjacency relation is the same sample-level adjacency as in
  \cref{def:full-transcript-zcdp}.
\end{definition}

The lower-bound proof for \((\bm\epsilon,\bm\delta)\)-DP public transcripts has
the same structure as the zCDP proof.
The posterior factorization, clientwise score decomposition, and van Trees
steps do not use the specific privacy divergence.
Only the single-client information contraction changes: the zCDP
\(\chi^2\)-comparison in
\cref{lem:single-client-matrix-zcdp-information} is replaced by a truncated
comparison for approximate DP.
Conditioning on \(X^{(-l)}\) then gives the corresponding federated contraction
exactly as in the proof of
\cref{thm:heterogeneous-matrix-information-contraction}.

For the rest of this section, use the notation of
\cref{lem:single-client-matrix-zcdp-information}.
Thus \(X_1,\ldots,X_n\) are independent observations,
\[
  S(\theta)=\sum_{i=1}^n s_i(X_i;\theta),
  \qquad
  J(\theta)=\E_\theta[S(\theta)S(\theta)^\T],
\]
and \(T=M(X)\) is a \(\theta\)-independent statistic.
Here \(T\) is assumed to be generated by a sample-level
\((\epsilon,\delta)\)-DP mechanism.
Fix \(\theta\), let all expectations be under \(P_\theta\), and write
\[
  U(T;\theta)=\E_\theta[S(\theta)\mid T],
  \qquad
  A(\theta)=\E_\theta[U(T;\theta)U(T;\theta)^\T],
\]
\[
  a=\Tr\zk{J(\theta)^\dagger A(\theta)} ,
  \qquad
  W(t;\theta)=J(\theta)^\dagger U(t;\theta).
\]
For each \(i\), let \(X_i'\), \(X^{(i)}\), and \(T^{(i)}=M(X^{(i)})\) be the
single-sample replacement variables used in
\cref{lem:single-client-matrix-zcdp-information}, and define
\[
  Z_i(t;\theta)=\ang{W(t;\theta),s_i(X_i;\theta)} .
\]
The \(\delta\)-part of approximate DP is controlled through the truncated tail
term, as in score attack lower bounds
\citep{cai2021_CostPrivacy,cai2023_ScoreAttack}.

\begin{lemma}[Approximate DP comparison with truncation]
  \label{lem:approx-dp-truncated-comparison}
  Let \(P\) and \(Q\) satisfy
  \[
    P(B)\le e^\epsilon Q(B)+\delta,
    \qquad
    Q(B)\le e^\epsilon P(B)+\delta
  \]
  for every measurable \(B\).
  If \(0\le\epsilon\le\epsilon_0\), then for every integrable real function
  \(f\) and every \(\tau>0\),
  \[
    \abs{\E_P f-\E_Q f}
    \le
    C_{\epsilon_0}\epsilon\,\E_Q\abs{f}
    +
    2\delta\tau
    +
    \E_P\zk{\abs{f}\mathbf 1\{\abs{f}>\tau\}}
    +
    \E_Q\zk{\abs{f}\mathbf 1\{\abs{f}>\tau\}} .
  \]
\end{lemma}

\begin{proof}
  Let \(f_\tau=(-\tau)\vee(f\wedge\tau)\).
  For any \(0\le h\le\tau\),
  \[
    \E_P h
    =
    \int_0^\tau P(h>t)\dd t
    \le
    e^\epsilon\E_Q h+\delta\tau ,
  \]
  while the reverse comparison gives
  \[
    \E_P h
    \ge
    e^{-\epsilon}\E_Q h-e^{-\epsilon}\delta\tau .
  \]
  Applying these two inequalities to the positive and negative parts of
  \(f_\tau\) yields
  \[
    \abs{\E_P f_\tau-\E_Q f_\tau}
    \le
    C_{\epsilon_0}\epsilon\,\E_Q\abs{f_\tau}
    +
    2\delta\tau .
  \]
  The tail bound follows from
  \(\abs{f-f_\tau}\le\abs{f}\mathbf 1\{\abs{f}>\tau\}\).
\end{proof}

\begin{lemma}[Single-client approximate DP contraction]
  \label{lem:single-client-approx-dp-information}
  Suppose \(0\le\epsilon\le\epsilon_0\).
  If there exist \(\tau_i>0\) and \(r_i\ge0\) such that
  \[
    \E\zk{
      \abs{Z_i(T;\theta)}
      \mathbf 1\{\abs{Z_i(T;\theta)}>\tau_i\}
    }
    +
    \E\zk{
      \abs{Z_i(T^{(i)};\theta)}
      \mathbf 1\{\abs{Z_i(T^{(i)};\theta)}>\tau_i\}
    }
    \le r_i
  \]
  for \(i=1,\ldots,n\), then
  \[
    \Tr\zk{J(\theta)^\dagger A(\theta)}
    \le
    C_{\epsilon_0}
    \left[
      n(e^\epsilon-1)^2
      +
      \sum_{i=1}^n(\delta\tau_i+r_i)
    \right].
  \]
\end{lemma}

\begin{proof}
  It suffices to consider \(J(\theta)\succ0\); the singular case is handled by
  restricting to \(\operatorname{range}(J(\theta))\), as in
  \cref{lem:single-client-matrix-zcdp-information}.
  The proof of \cref{lem:single-client-matrix-zcdp-information} gives
  \(0\preceq A\preceq J\) and
  \[
    a
    =
    \sum_{i=1}^n \E Z_i(T;\theta).
  \]
  Conditional on \((X,X^{(i)})\), the laws of \(T=M(X)\) and
  \(T^{(i)}=M(X^{(i)})\) satisfy the two-sided
  \((\epsilon,\delta)\)-DP comparison.
  Since \(T^{(i)}\) is independent of \(X_i\) and \(\E s_i(X_i)=0\),
  \(\E Z_i(T^{(i)};\theta)=0\).
  Applying \cref{lem:approx-dp-truncated-comparison} conditionally gives
  \[
    \abs{\E Z_i(T;\theta)}
    \le
    C_{\epsilon_0}\epsilon\,\E\abs{Z_i(T^{(i)};\theta)}
    +
    2\delta\tau_i+r_i .
  \]
  With
  \[
    b_i=\Tr\left(J^{-1}AJ^{-1}I_i(\theta)\right),
  \]
  the independence of \(T^{(i)}\) and \(X_i\) gives
  \(\E\abs{Z_i(T^{(i)};\theta)}\le\sqrt{b_i}\), and the proof of
  \cref{lem:single-client-matrix-zcdp-information} gives
  \(\sum_i b_i=a\).
  Hence
  \[
    a
    \le
    C_{\epsilon_0}\epsilon\sqrt{na}
    +
    \sum_{i=1}^n(2\delta\tau_i+r_i).
  \]
  Solving the quadratic inequality proves the result, using
  \(\epsilon\asymp e^\epsilon-1\) for \(0\le\epsilon\le\epsilon_0\).
\end{proof}

To control the truncation tail, we introduce a convenient sub-Gaussian condition on the whitened score:
\begin{equation}
  \label{eq:homogeneous-subgaussian-score}
  \norm{u^\T I_x(\theta)^{-1/2}s(X;\theta)}_{\psi_2}\le K,
  \qquad
  \forall u \text{ s.t. } \norm{u}_2=1 .
\end{equation}
For the Gaussian mean model \(X\sim\mca{N}(\theta,\sigma^2 I_p)\),
\(I_x(\theta)=\sigma^{-2}I_p\) and
\(u^\T I_x(\theta)^{-1/2}s(X;\theta)=\sigma^{-1}u^\T(X-\theta)\), which is
standard normal for every unit \(u\).
Thus \cref{eq:homogeneous-subgaussian-score} holds uniformly with a universal
constant \(K\).

\begin{corollary}[Sub-Gaussian score sufficient condition]
  \label{cor:subgaussian-score-approx-dp-contraction}
  Let \(r=\operatorname{rank}(J(\theta))\) and \(0<\delta\le e^{-1}\).
  Suppose that, on \(\operatorname{range}(J(\theta))\),
  \(J(\theta)^{\dagger/2}S(\theta)\) is \(K\)-sub-Gaussian and
  \(J(\theta)^{\dagger/2}s_i(X_i;\theta)\) is \(K_i\)-sub-Gaussian for each
  \(i\).
  Then
  \[
    \Tr\zk{J(\theta)^\dagger A(\theta)}
    \le
    C_{\epsilon_0,K}
    \left[
      n(e^\epsilon-1)^2
      +
      \delta\,r\log(1/\delta)\sum_{i=1}^n K_i
    \right].
  \]
  In the i.i.d.\ homogeneous case \(J(\theta)=nI_x(\theta)\), if
  \cref{eq:homogeneous-subgaussian-score} holds, this reduces to
  \[
    \Tr\zk{J(\theta)^\dagger A(\theta)}
    \le
    C_{\epsilon_0,K}
    \left[
      n(e^\epsilon-1)^2
      +
      \delta\,p\sqrt n\log(1/\delta)
    \right].
  \]
  Thus the contribution of \( \delta \) vanishes relative to the \( \epsilon \) term whenever
  \[
    \delta\,p\log(1/\delta)\lesssim \epsilon^2\sqrt n
  \]
  at the finite-dimensional resolution \(p\).
\end{corollary}

\begin{proof}
  Jensen's inequality transfers the sub-Gaussian bound from the whitened full
  score to the whitened posterior score. Hence
  \[
    \ang{J^{\dagger/2}U(T;\theta),J^{\dagger/2}s_i(X_i;\theta)}
  \]
  and the corresponding variable with \(T^{(i)}\) in place of \(T\) have
  sub-exponential tails with scale at most \(C_K rK_i\).
  Taking \(\tau_i=C_K rK_i\log(1/\delta)\) gives
  \[
    r_i
    \le
    C_K\delta\,rK_i\log(1/\delta).
  \]
  The first inequality follows from
  \cref{lem:single-client-approx-dp-information}.
  In the homogeneous case,
  \(J^{-1/2}s_i=n^{-1/2}I_x^{-1/2}s_i\), so \(K_i\le K/\sqrt n\), and
  \(r=p\).
\end{proof}

\begin{theorem}[Scalar approximate DP federated van Trees inequality]
  \label{thm:homogeneous-approx-dp-vantrees}
  Under model \cref{eq:homogeneous-independent-client-model} and
  \cref{ass:van-trees-regularity}, suppose \(I_x(\theta)\succ0\)
  and \cref{eq:homogeneous-subgaussian-score} holds
  for all
  \(\theta\in\supp(\pi)\).
  Let \(G\) be the complete public transcript of an
  \((\bm\epsilon,\bm\delta)\)-DP federated protocol in the sense of
  \cref{def:full-transcript-approx-dp}, with
  \(0\le\epsilon_l\le\epsilon_0\) and \(0<\delta_l\le e^{-1}\).
  Define
  \[
    \Gamma_l(\theta)
    =
    \left[
      \xk{ \epsilon_l^2 n_l^2
        +
        \delta_l p n_l^{3/2}\log(1/\delta_l)}
      \norm{I_x(\theta)}_{\mathrm{op}}
    \right]
    \wedge
    n_l\Tr I_x(\theta).
  \]
  Then every estimator \(\hat\theta=\hat\theta(G)\) satisfies
  \[
    \E_\pi\E_\theta\norm{\hat\theta-\theta}_2^2
    \gtrsim
    \frac{p^2}{
      \int_\Theta \sum_{l=1}^m \Gamma_l(\theta)\dd\pi(\theta)
      +
      \Tr J_\pi
    }.
  \]
\end{theorem}

\begin{proof}
  The proof is the same as that of \cref{thm:general-fed-dp-vantrees}, with the
  zCDP transcript contraction in
  \cref{lem:transcript-information-contraction} replaced by the following
  approximate-DP contraction.
  Conditioning on \(X^{(-l)}\) and applying
  \cref{cor:subgaussian-score-approx-dp-contraction} gives, for the \(l\)th
  client score contribution \(A_l(\theta)\) used in that proof,
  \[
    \Tr A_l(\theta)
    \lesssim
    \left\{
      n_l^2\epsilon_l^2
      +
      \delta_l p n_l^{3/2}\log(1/\delta_l)
    \right\}
    \norm{I_x(\theta)}_{\mathrm{op}}.
  \]
  The raw-information bound \(\Tr A_l(\theta)\le n_l\Tr I_x(\theta)\) is
  unchanged; summing over clients and applying the same multivariate van Trees
  inequality proves the theorem.
\end{proof}

\begin{remark}[Weaker tail conditions]
  The sub-Gaussian condition in \cref{eq:homogeneous-subgaussian-score} is not
  essential.
  It can be replaced by any score tail bound that makes the truncation remainder
  \(\sum_i(\delta\tau_i+r_i)\) in
  \cref{lem:single-client-approx-dp-information} of smaller order than the
  privacy term required by the target lower bound.
\end{remark}
 \clearpage
\section{Application Lower Bound Derivations}
\label{app:application-derivations}

This appendix gives the lower bound calculations behind
\cref{sec:applications}.
We repeatedly use the following minimax reduction.
If a prior \(\pi\) is supported in the parameter space, then
\[
  \inf_{\hat\theta\in\mca{M}_{\bm\rho}}
  \sup_{\theta\in\Theta}
  \E_\theta L(\hat\theta,\theta)
  \ge
  \inf_{\hat\theta\in\mca{M}_{\bm\rho}}
  \E_\pi\E_\theta L(\hat\theta,\theta).
\]
For rectangular finite-dimensional submodels we use the cosine prior with
density
\[
  p_0(t)=a^{-1}\cos^2\left(\frac{\pi t}{2a}\right),
  \qquad t\in[-a,a],
\]
and \(p_0(t)=0\) otherwise.
It has one-dimensional Fisher information \(\pi^2/a^2\).
Thus a \(p\)-fold product version has prior Fisher matrix
\((\pi^2/a^2)I_p\) and trace \(p\pi^2/a^2\).

\subsection{Mean estimation}
\label{app:deriv-mean-estimation}

\subsubsection{Gaussian mean estimation}
\label{app:deriv-gaussian-mean}

For the Gaussian mean model in
\cref{eq:app-gaussian-mean-model}, the score of one observation has
Fisher information
\[
  I_x(\theta)=\sigma^{-2}I_d,
  \qquad
  \norm{I_x(\theta)}_{\mathrm{op}}=\sigma^{-2},
  \qquad
  \Tr I_x(\theta)=d\sigma^{-2}.
\]
Applying \cref{cor:bounded-rho-fed-dp-vantrees} with the \(d\)-fold cosine
prior supported on \([-1,1]^d\) gives
\[
  \E_\pi\E_\theta\norm{\hat\theta-\theta}_2^2
  \gtrsim
  \frac{d^2}{
    \sigma^{-2}
    \sum_{l=1}^m
    \left(\rho_l n_l^2\wedge d n_l\right)
    +
    d
  }.
\]
Since
\[
  \rho_l n_l^2\wedge d n_l
  =
  d^2
  \left(
    \frac{d}{n_l}
    \vee
    \frac{d^2}{\rho_l n_l^2}
  \right)^{-1},
\]
the preceding display is equivalent, up to constants, to
\[
  \inf_{\hat\theta \in \mca{M}_{\bm{\rho}}}
  \sup_{\theta\in [-1,1]^d}
  \E_\theta\norm{\hat\theta-\theta}_2^2
  \gtrsim
  d
  \wedge
  \sigma^2
  \zk{
    \sum_{l=1}^m
    \left(
      \frac{d}{n_l}
      \vee
      \frac{d^2}{\rho_l n_l^2}
    \right)^{-1}
  }^{-1}.
\]
The minimax bound follows because the prior is supported in \([-1,1]^d\).

\subsubsection{Missing features and coordinate coverage}
\label{app:deriv-balanced-coordinate-mean}

For the missing feature model, the Fisher information of one sample for client
\(l\) is
\[
  I_l(\theta)=\sigma^{-2}P_l,
  \qquad
  J_l(\theta)=n\sigma^{-2}P_l .
\]
Thus the setting is exactly the projection experiment discussed in
\cref{subsec:general-vantrees-discussion}, with \(p=d\) and client projection
rank \(s\).
Under balanced coverage,
\[
  \sum_{l=1}^m P_l=qI_d,
  \qquad
  ms=dq .
\]
\Cref{cor:balanced-coordinate-coverage} with the product cosine prior on
\([-1,1]^d\), followed by the minimax reduction, gives
\[
  \inf_{\hat\theta\in \mca{M}_{\bm{\rho}}}
  \sup_{\theta\in [-1,1]^d}
  \E_\theta\norm{\hat\theta-\theta}_2^2
  \gtrsim
  \frac{d}{
    1+
    \sigma^{-2}qn
    \left(1\wedge\frac{\kappa n}{s}\right)
  },
  \qquad
  \kappa=e^{2\rho}-1 .
\]
For bounded \(\rho\), \(\kappa\asymp\rho\).
This is equivalent, up to constants, to
\[
  \inf_{\hat\theta\in \mca{M}_{\bm{\rho}}}
  \sup_{\theta\in [-1,1]^d}
  \E_\theta\norm{\hat\theta-\theta}_2^2
  \gtrsim
  d
  \wedge
  \sigma^2
  \left(
    \frac{d}{qn}
    \vee
    \frac{ds}{q\rho n^2}
  \right).
\]
This calculation also shows the range obstruction: if a coordinate has zero
coverage, then the transcript information matrix has no component in that
coordinate and the risk there is controlled only by the prior.

\subsection{Linear regression}
\label{app:deriv-linear-regression-section}

\subsubsection{Homogeneous random design}
\label{app:deriv-linear-regression}

For client \(l\) in random design linear regression,
\[
  Y=(Z^{(l)})^\T\theta+\xi,
  \qquad
  \xi\sim\mca{N}(0,\sigma^2),
\]
the joint observation \(X=(Z^{(l)},Y)\) has score
\(\sigma^{-2}Z^{(l)}(Y-(Z^{(l)})^\T\theta)\), and hence
\[
  I_{x,l}(\theta)=\sigma^{-2}\E Z^{(l)}(Z^{(l)})^\T .
\]
If \(\E Z^{(l)}(Z^{(l)})^\T\preceq c_1 I_d\) uniformly in \(l\), then
\[
  \norm{I_{x,l}(\theta)}_{\mathrm{op}}\le c_1\sigma^{-2},
  \qquad
  \Tr I_{x,l}(\theta)\le c_1d\sigma^{-2}.
\]
The trace consequence of the heterogeneous matrix contraction therefore gives
the same denominator as in \cref{app:deriv-gaussian-mean}, up to constants
depending on \(c_1\), when applied with the \(d\)-dimensional cosine prior on
\([-1,1]^d\).
Therefore
\[
  \inf_{\hat\theta\in\mca{M}_{\bm{\rho}}}
  \sup_{\theta\in [-1,1]^d}
  \E_\theta\norm{\hat\theta-\theta}_2^2
  \gtrsim
  d
  \wedge
  \sigma^2
  \zk{
    \sum_{l=1}^m
    \left(
      \frac{d}{n_l}
      \vee
      \frac{d^2}{\rho_l n_l^2}
    \right)^{-1}
  }^{-1}.
\]
The lower eigenvalue condition in
\cref{eq:app-homogeneous-linear-regression-design} is not needed for the
lower bound itself; it records the usual well conditioned regime in which this
lower bound has the same order as the corresponding nonprivate and private
upper rates.

\subsubsection{Anisotropic and heterogeneous designs}
\label{app:deriv-heterogeneous-linear-regression}

For client \(l\),
\[
  I_l(\theta)=\sigma_l^{-2}\Sigma_l,
  \qquad
  J_l(\theta)=n_l\sigma_l^{-2}\Sigma_l .
\]
Substituting this into \cref{thm:matrix-fed-dp-vantrees} yields
\[
  J_l^{1/2}B_lJ_l^{1/2}
  =
  n_l\sigma_l^{-2}\Sigma_l^{1/2}B_l\Sigma_l^{1/2},
\]
where
\[
  0\preceq B_l\preceq P_{\Sigma_l},
  \qquad
  \Tr B_l\le (e^{2\rho_l}-1)n_l.
\]
Therefore
\[
  \E_\pi\E_\theta\norm{\hat\theta-\theta}_2^2
  \ge
  \inf_{\{B_l\}}
  \Tr
  \left[
    J_\pi+
    \sum_{l=1}^m
    n_l\sigma_l^{-2}\Sigma_l^{1/2}B_l\Sigma_l^{1/2}
  \right]^{-1}.
\]
For \(\Theta=[-1,1]^d\), the product cosine prior has
\(J_\pi=c I_d\) for a universal constant \(c>0\), giving the display in
\cref{eq:app-heterogeneous-linear-regression-bound} by the minimax reduction.
Each summand is supported on \(\operatorname{range}(\Sigma_l)\).
Consequently the public transcript cannot reduce posterior uncertainty in
directions orthogonal to
\(\operatorname{span}\{\operatorname{range}(\Sigma_l):1\le l\le m\}\).
This is the directional obstruction that a scalar trace denominator discards.

\subsubsection{Target domain prediction}
\label{app:deriv-target-domain-prediction}

Take \(P_0\) to be the projection onto an \(r\)-dimensional target subspace and
consider
\[
  L_0(\hat\theta,\theta)
  =
  (\hat\theta-\theta)^\T P_0(\hat\theta-\theta).
\]
Only the \(q\) clients with \(\Sigma_l=P_0\) contribute to
\(L_0\); all other clients satisfy \(P_0\Sigma_lP_0=0\) and hence
contribute no target information.
Within the target subspace, each relevant client has raw Fisher information
\(n\sigma^{-2}I_r\) and a whitened budget \(\kappa n\), where
\(\kappa=e^{2\rho}-1\).
For the minimax risk over \(B_2(0,1)\), choose a product cosine prior on a
coordinate cube of radius \(r^{-1/2}\) inside the target subspace.
Its Fisher matrix in the target coordinates is bounded by a constant multiple
of \(rI_r\).
The prior is supported in \(B_2(0,1)\).
Applying the balanced coordinate calculation with dimension \(r\), coverage
\(q\), and rank \(r\) gives
\[
  \inf_{\hat\theta\in \mca{M}_{\bm{\rho}}}
  \sup_{\theta\in B_2(0,1)}
  \E_\theta L_0(\hat\theta,\theta)
  \gtrsim
  1
  \wedge
  \sigma^2
  \left(
    \frac{r}{qn}
    \vee
    \frac{r^2}{q\kappa n^2}
  \right).
\]
For bounded \(\rho\), \(\kappa\asymp\rho\), which gives the rate displayed in
\cref{eq:app-target-domain-prediction-rate}.

\subsection{Nonparametric regression}
\label{app:deriv-nonparametric-regression}

Let \(\phi_1,\ldots,\phi_p\) be an orthonormal system in \(L^2([0,1]^d)\)
adapted to the Sobolev norm, and consider the finite-dimensional submodel
\[
  f_\theta(u)=\sum_{j=1}^p\theta_j\phi_j(u).
\]
Choose a product cosine prior on the coefficients with radius
\[
  r_p\asymp R p^{-1/2-\alpha/d}.
\]
Then the prior is supported in \(H^\alpha(R)\), and
\[
  \Tr J_{\pi,p}
  \asymp
  p r_p^{-2}
  \asymp
  R^{-2}p^{2+2\alpha/d}.
\]
The design density is bounded above and below, so the Fisher information of
one observation in the submodel satisfies
\[
  \norm{I_x(\theta)}_{\mathrm{op}}\lesssim \sigma^{-2},
  \qquad
  \Tr I_x(\theta)\lesssim \sigma^{-2}p .
\]
\Cref{cor:bounded-rho-fed-dp-vantrees} gives, for each fixed resolution \(p\),
\[
  \E_\pi\E_f\norm{\hat f-f}_{L^2}^2
  \gtrsim
  \frac{p^2}{
    R^{-2}p^{2+2\alpha/d}
    +
    \sigma^{-2}
    \sum_{l=1}^m
    \left(
      \rho_l n_l^2
      \wedge
      p n_l
    \right)
  }.
\]
Dividing numerator and denominator by \(p^2\) gives
\[
  \E_\pi\E_f\norm{\hat f-f}_{L^2}^2
  \gtrsim
  \zk{
    \frac{p^{2\alpha/d}}{R^2}
    +
    \frac{1}{\sigma^2}
    \sum_{l=1}^m
    \left(
      \frac{p}{n_l}
      \vee
      \frac{p^2}{\rho_l n_l^2}
    \right)^{-1}
  }^{-1}.
\]
Since the prior is supported in \(H^\alpha(R)\), the minimax reduction applies
for each \(p\).
Taking the supremum over \(p\) proves the heterogeneous display.

For homogeneous clients, the nonprivate component is obtained by balancing
\[
  R^{-2}p^{2\alpha/d}
  \asymp
  \frac{mn}{\sigma^2p},
\]
which yields
\[
  R^{\frac{2d}{2\alpha+d}}
  \left(
    \frac{\sigma^2}{mn}
  \right)^{\frac{2\alpha}{2\alpha+d}}.
\]
The privacy component is obtained by balancing
\[
  R^{-2}p^{2\alpha/d}
  \asymp
  \frac{\rho mn^2}{\sigma^2p^2},
\]
which yields
\[
  R^{\frac{2d}{\alpha+d}}
  \left(
    \frac{\sigma^2}{\rho mn^2}
  \right)^{\frac{\alpha}{\alpha+d}}.
\]
Since a lower bound by the larger of finitely many components is equivalent,
up to a universal constant, to a lower bound by their sum, the homogeneous
rate displayed in \cref{eq:app-nonparametric-regression-homogeneous-rate}
follows.

\subsection{Functional mean estimation}
\label{app:deriv-functional-mean}

The privacy unit is one entire discretely observed curve.
The finite-dimensional reductions below use localized wavelet submodels
\[
  f_\theta=\sum_{r=1}^p\theta_r\phi_r
\]
with coefficient radius \(p^{-\alpha-1/2}\).
For these submodels,
\[
  \Tr J_\pi\asymp p^{2\alpha+2},
\]
so the Sobolev truncation contribution after the \(p^2\) van Trees numerator is
\(p^{2\alpha}\).

\subsubsection{Independent design}
\label{app:deriv-functional-independent}

For the independent design lower bound, take the measurement locations
\(T_{ij}^{(l)}\) to be private observations generated independently from a
density on \([0,1]\) bounded above and below.
Use the finite-dimensional Gaussian curve submodel
\[
  X_i^{(l)}(t)
  =
  f_\theta(t)
  +
  \sum_{r=1}^p \eta_{ir}^{(l)}\phi_r(t),
  \qquad
  \eta_i^{(l)}\sim \mca{N}(0,\tau_p^2 I_p),
  \qquad
  \tau_p^2\asymp p^{-2\alpha-1}.
\]
This variation at the curve level is admissible because
\[
  \E\norm{X_i^{(l)}-f_\theta}_{H^\alpha}^2
  \lesssim
  p^{2\alpha}\cdot p\cdot \tau_p^2
  \lesssim 1.
\]
For a fixed client \(l\), write
\[
  \varphi(t)=(\phi_1(t),\ldots,\phi_p(t))^\T,
  \qquad
  \Phi_T=
  \begin{pmatrix}
    \varphi(T_{i1}^{(l)})^\T\\
    \vdots\\
    \varphi(T_{ik_l}^{(l)})^\T
  \end{pmatrix}.
\]
Conditionally on the location vector
\(T=(T_{i1}^{(l)},\ldots,T_{ik_l}^{(l)})\), the observed vector
\(Y_i^{(l)}=(Y_{i1}^{(l)},\ldots,Y_{ik_l}^{(l)})^\T\) satisfies
\[
  Y_i^{(l)}
  =
  \Phi_T\theta+\Phi_T\eta_i^{(l)}+\xi_i^{(l)},
  \qquad
  \xi_i^{(l)}\sim\mca{N}(0,I_{k_l}),
\]
and therefore
\[
  Y_i^{(l)}\mid T,\theta
  \sim
  \mca{N}(\Phi_T\theta,C_T),
  \qquad
  C_T=I_{k_l}+\tau_p^2\Phi_T\Phi_T^\T .
\]
The law of \(T\) and the covariance \(C_T\) do not depend on \(\theta\).
Hence the Fisher information matrix of one independently designed curve is
\begin{equation}
  \label{eq:app-functional-independent-one-curve-information}
  I_{l,p}^{\mathrm{ind}}(\theta)
  =
  \E_T\left[\Phi_T^\T C_T^{-1}\Phi_T\right].
\end{equation}
This is the final Fisher information for one curve used below.
It has two useful upper envelopes.
First, since \(C_T\succeq I_{k_l}\),
\[
  I_{l,p}^{\mathrm{ind}}(\theta)
  \preceq
  \E_T\Phi_T^\T\Phi_T
  =
  k_l\E_T[\varphi(T)\varphi(T)^\T]
  \preceq C k_l I_p,
\]
where the last inequality uses the \(L^2\)-normalization of the localized
wavelets and the bounded design density.
Second, the Gaussian variation at the curve level gives
\[
  \Phi_T^\T C_T^{-1}\Phi_T
  =
  \tau_p^{-2}
  \left[
    I_p-(I_p+\tau_p^2\Phi_T^\T\Phi_T)^{-1}
  \right].
\]
Consequently
\[
  I_{l,p}^{\mathrm{ind}}(\theta)
  \preceq
  \tau_p^{-2}I_p
  \preceq
  C p^{2\alpha+1}I_p.
\]
Combining the two envelopes gives the explicit bounds
\begin{equation}
  \label{eq:app-functional-independent-one-curve-information-bound}
  \norm{I_{l,p}^{\mathrm{ind}}(\theta)}_{\mathrm{op}}
  \lesssim
  k_l\wedge p^{2\alpha+1},
  \qquad
  \Tr I_{l,p}^{\mathrm{ind}}(\theta)
  \lesssim
  k_lp\wedge p^{2\alpha+2}.
\end{equation}
Applying the trace contraction to both upper bounds and retaining the smaller
allowed transcript information gives, after dividing by the \(p^2\) numerator,
the client contribution
\[
  \frac{k_l n_l}{p}
  \wedge
  \frac{\rho_l k_l n_l^2}{p^2}
  \wedge
  p^{2\alpha}n_l
  \wedge
  p^{2\alpha-1}\rho_l n_l^2 .
\]
Combining this contribution with the truncation term \(p^{2\alpha}\) proves
\[
  \begin{aligned}
    &\inf_{\hat f\in\mca{M}_{\bm{\rho}}}
    \sup_{f\in H^\alpha(R)}
    \E_f\norm{\hat f-f}_{L^2}^2 \\
    &\qquad\gtrsim
    \sup_{p\ge1}
    \left\{
      p^{2\alpha}
      +
      \sum_{l=1}^m
      \left(
        \frac{k_l n_l}{p}
        \wedge
        \frac{\rho_l k_l n_l^2}{p^2}
        \wedge
        p^{2\alpha}n_l
        \wedge
        p^{2\alpha-1}\rho_l n_l^2
      \right)
    \right\}^{-1}.
  \end{aligned}
\]

In the homogeneous case, the four displayed rate components are obtained as
follows.
The sample size component at the curve level is obtained by taking \(p=1\),
giving \(1/(mn)\).
Balancing \(p^{2\alpha}\) with \(mkn/p\) gives
\[
  p^{2\alpha+1}\asymp mkn,
  \qquad
  \text{risk}\asymp (mkn)^{-2\alpha/(2\alpha+1)}.
\]
Balancing \(p^{2\alpha}\) with \(\rho m k n^2/p^2\) gives
\[
  p^{2\alpha+2}\asymp \rho m k n^2,
  \qquad
  \text{risk}\asymp (\rho m k n^2)^{-\alpha/(\alpha+1)}.
\]
Finally, the privacy bottleneck at low frequencies
\(\rho m n^2p^{2\alpha-1}\) at \(p=1\) gives \(1/(\rho m n^2)\).
Together these yield the homogeneous lower bound displayed in
\cref{eq:app-functional-independent-homogeneous-rate}.

\subsubsection{Common design}
\label{app:deriv-functional-common}

We give the common design construction explicitly.
For clarity take a regular public grid \(t_j=j/k\), \(j=1,\ldots,k\); the same
argument only uses quasi uniform grid spacing.
Fix \(1\le p\le k\) and partition the grid into \(p\) consecutive blocks
\[
  B_r=\{j:\ t_j\in I_r\},
  \qquad
  I_r=\left[\frac{r-1}{p},\frac{r}{p}\right],
  \qquad r=1,\ldots,p .
\]
Each block contains order \(k/p\) grid points.
Choose a fixed compactly supported smooth wavelet \(\psi\) with
\(\norm{\psi}_{L^2}=1\), and define
\[
  \phi_r(t)=p^{1/2}\psi(pt-r+1),
  \qquad r=1,\ldots,p,
\]
after a harmless translation of the supports inside \(I_r\).
The functions \(\phi_r\) have disjoint supports, are orthonormal in \(L^2\),
and satisfy
\[
  \norm{\phi_r}_{H^\alpha}\lesssim p^\alpha .
\]
Hence the submodel
\[
  f_\theta(t)=\sum_{r=1}^p \theta_r\phi_r(t),
  \qquad
  \abs{\theta_r}\lesssim p^{-\alpha-1/2},
\]
is contained in the Sobolev ball.
As in the independent design calculation, the product prior on these
coefficients has
\[
  \Tr J_\pi\asymp p^{2\alpha+2}.
\]

Let \(v_r=(\phi_r(t_1),\ldots,\phi_r(t_k))^\T\in\R^k\), and let
\(\Phi\) be the \(p\times k\) matrix with rows \(v_r^\T\).
By disjoint support and the grid Riemann sum estimate,
\[
  v_r^\T v_s=0\quad(r\ne s),
  \qquad
  \norm{v_r}_2^2\asymp k .
\]
Use the Gaussian submodel for the mean curve
\[
  X_i(t)
  =
  f_\theta(t)
  +
  \sum_{r=1}^p \eta_{ir}\phi_r(t),
  \qquad
  \eta_{ir}\stackrel{\mathrm{i.i.d.}}{\sim}\mca{N}(0,\tau_p^2),
  \qquad
  \tau_p^2=c/p ,
\]
with \(c>0\) fixed.
The random curve variation has bounded \(L^2\) energy since
\(\E\norm{X_i-f_\theta}_{L^2}^2=p\tau_p^2=c\).
On the common grid, the observed vector
\(Y_i=(Y_{i1},\ldots,Y_{ik})^\T\) satisfies
\[
  Y_i
  =
  \Phi^\T\theta+\Phi^\T\eta_i+\xi_i,
  \qquad
  \xi_i\sim\mca{N}(0,I_k),
\]
and hence
\[
  Y_i\mid\theta
  \sim
  \mca{N}(\Phi^\T\theta,C_p),
  \qquad
  C_p=I_k+\tau_p^2\Phi^\T\Phi
  =
  I_k+\tau_p^2\sum_{r=1}^p v_rv_r^\T .
\]
The covariance does not depend on \(\theta\), so the final Fisher
information matrix for one curve in the coefficient parameter is
\begin{equation}
  \label{eq:app-functional-common-one-curve-information}
  I_p^{\mathrm{com}}(\theta)
  =
  \Phi C_p^{-1}\Phi^\T .
\end{equation}
The vectors \(v_r\) are mutually orthogonal.
Therefore the Sherman--Morrison formula applied on each support block gives
\[
  v_r^\T C_p^{-1}v_r
  =
  \frac{\norm{v_r}_2^2}{
    1+\tau_p^2\norm{v_r}_2^2
  }
  \lesssim
  \frac{k}{1+k/p}
  \lesssim p,
  \qquad r=1,\ldots,p,
\]
and \(v_r^\T C_p^{-1}v_s=0\) for \(r\ne s\).
Consequently the information matrix in
\cref{eq:app-functional-common-one-curve-information} is explicitly
\begin{equation}
  \label{eq:app-functional-common-one-curve-information-diagonal}
  I_p^{\mathrm{com}}(\theta)
  =
  \operatorname{diag}
  \left(
    \frac{\norm{v_1}_2^2}{1+\tau_p^2\norm{v_1}_2^2},
    \ldots,
    \frac{\norm{v_p}_2^2}{1+\tau_p^2\norm{v_p}_2^2}
  \right)
  \preceq CpI_p .
\end{equation}
Consequently,
\begin{equation}
  \label{eq:app-functional-common-one-curve-information-bound}
  \norm{I_p^{\mathrm{com}}(\theta)}_{\mathrm{op}}\lesssim p,
  \qquad
  \Tr I_p^{\mathrm{com}}(\theta)\lesssim p^2 .
\end{equation}
After dividing the transcript information terms by \(p^2\), client \(l\)'s
contribution is
\[
  n_l
  \wedge
  \frac{\rho_l n_l^2}{p}.
\]
The public grid also leaves the usual discretization obstruction
\(k^{-2\alpha}\): one may put an \(H^\alpha\)-admissible bump between adjacent
grid points, or use wavelet levels above the grid resolution, so that two mean
functions agree on the public grid while their squared \(L^2\)-distance is of
order \(k^{-2\alpha}\).
Therefore
\[
  \inf_{\hat f\in\mca{M}_{\bm{\rho}}}
  \sup_{f\in H^\alpha(R)}
  \E_f\norm{\hat f-f}_{L^2}^2
  \gtrsim
  k^{-2\alpha}
  +
  \sup_{1\le p\le k}
  \left\{
    p^{2\alpha}
    +
    \sum_{l=1}^m
    \left(
      n_l
      \wedge
      \frac{\rho_l n_l^2}{p}
    \right)
  \right\}^{-1}.
\]

For homogeneous clients, taking \(p=1\) gives the \(1/(mn)\) component when the
raw sample size term is active.
The privacy component is obtained by balancing
\[
  p^{2\alpha}
  \asymp
  \frac{\rho m n^2}{p},
  \qquad
  p^{2\alpha+1}\asymp \rho m n^2,
\]
which yields
\[
  (\rho m n^2)^{-2\alpha/(2\alpha+1)}.
\]
If the optimizing \(p\) exceeds the grid resolution \(k\), the discretization
term \(k^{-2\alpha}\) is active instead.
This proves the rate displayed in
\cref{eq:app-functional-common-homogeneous-rate}.

\end{document}